\pdfoutput=1
\documentclass[12pt]{article}
\usepackage{graphicx}

\usepackage{wrapfig}
\usepackage{natbib}
\usepackage[colorlinks,citecolor=blue,urlcolor=blue]{hyperref}
\usepackage{booktabs}
\usepackage{pifont}

\usepackage{multirow}
\usepackage{algorithm}
\usepackage{algpseudocode}
\usepackage{amsmath, amssymb, amsthm, amsfonts}
\usepackage{mathtools}
\usepackage{bbm}
\usepackage{authblk}
\usepackage{array}
\usepackage{enumitem}
\usepackage{subcaption}
\usepackage{cleveref}
\usepackage{geometry}
\usepackage{xcolor}

\newcommand*{\eg}{\emph{e.g.}{}}
\newcommand*{\ie}{\emph{i.e.}{}}

\newcommand{\equalcontrib}{\textsuperscript{*}}

\theoremstyle{plain}
\newtheorem{theorem}{Theorem}
\newtheorem{prop}{Proposition}
\newtheorem{assumption}{Assumption}
\newtheorem{corollary}{Corollary}
\newtheorem{lemma}{Lemma}

\theoremstyle{remark}
\newtheorem{remark}{Remark}

\title{\bfseries
Learning Polyhedral Conformal Sets\\ for Robust Optimization
}

\author{Shuyi Chen\equalcontrib}
\author{Wenbin Zhou\equalcontrib}
\author{Shixiang Zhu}
\affil{Carnegie Mellon University}

\date{}

\begin{document}

\maketitle

\begingroup
\renewcommand\thefootnote{*}
\footnotetext{Equal contribution.}
\endgroup

\begin{abstract}
Robust optimization (RO) provides a principled framework for decision-making under uncertainty, but its performance critically depends on the choice of the uncertainty set. While large sets ensure reliability, they often lead to overly conservative decisions, whereas small sets risk excluding the true outcome. Recent data-driven approaches, particularly conformal prediction, offer finite-sample validity guarantees but remain largely task-agnostic, ignoring the downstream decision structure. In this paper, we propose a decision-aware conformal framework that learns uncertainty sets tailored to robust optimization objectives. Our approach parameterizes a flexible family of polyhedral sets via data-driven hyperplanes and learns their geometry by directly minimizing the induced robust loss, while preserving statistical validity through conformal calibration. To correct for data-dependent selection, we incorporate a recalibration step on an independent dataset to restore coverage. The resulting sets capture directional and anisotropic uncertainty aligned with the decision objective while remaining computationally tractable. We provide finite-sample coverage guarantees and bounds on the suboptimality gap to an oracle decision. This work bridges the gap between statistical validity and decision optimality, providing a principled framework for data-driven robust optimization.
\end{abstract}

\section{Introduction}

Robust optimization (RO) is a fundamental framework for decision-making under uncertainty~\citep{ben2009robust, ben1998robust}, with applications across operations research, finance, energy systems, supply chain management, and machine learning. RO prescribes decisions that perform well under worst-case realizations of uncertain parameters, hedging against ambiguity and model misspecification. This is particularly valuable in high-stakes settings: power grid operators must plan infrastructure upgrades under uncertain demand and renewable generation~\citep{6547161,ABDIN2022118032,zhou2024hierarchical}; supply chain managers must allocate inventory without knowing future disruptions~\citep{e175628f-da0b-3bc1-9746-ae30840a1bc5}; and healthcare systems must make treatment decisions under uncertain patient outcomes~\citep{mo2021learning}.

A central component of robust optimization is the specification of the uncertainty set, which characterizes plausible realizations of the unknown parameters. The quality of the resulting decision critically depends on this choice. A well-known challenge is the trade-off between \emph{conservativeness and performance}: overly large uncertainty sets provide strong protection against worst-case scenarios but often lead to overly pessimistic decisions, while overly small sets may exclude the true outcome, resulting in unreliable and potentially catastrophic decisions.

While the literature has extensively studied how to ensure robustness—often by constructing uncertainty sets that are sufficiently large to guarantee feasibility or coverage—it has paid comparatively less attention to the risk of over-conservativeness~\citep{zhou2025robustness, wang2025mean}. In practice, this imbalance can have significant consequences. For instance, in power systems planning, an overly conservative uncertainty set may trigger unnecessary infrastructure investments, leading to excessive costs for utilities and consumers~\citep{chen2026large}. In inventory management, it may result in systematic overstocking, tying up capital and increasing waste \citep{lim2017inventory}. Similarly, in portfolio optimization, overly conservative risk sets can lead to excessively defensive allocations that sacrifice substantial returns \citep{gregory2011robust}. These examples highlight the importance of striking the right balance: uncertainty sets should be large enough to ensure reliability, yet sufficiently tight and structured to avoid overly pessimistic decisions.

Recent data-driven approaches construct uncertainty sets via conformal prediction~\citep{yeh2024end, patel2024conformal}, which provides distribution-free, finite-sample guarantees that the true outcome lies in the set with high probability~\citep{vovk2005algorithmic}. However, standard conformal prediction is inherently \emph{task-agnostic}: it constructs sets based solely on predictive accuracy, without accounting for the downstream decision problem~\citep{pmlr-v206-wang23n}, so the resulting sets may be unnecessarily large or poorly shaped for decision-making. Methods that tighten conformal sets or adapt them to the data distribution~\citep{zheng2024generative} remain largely disconnected from the structure of the downstream optimization problem. Consequently, there remains a fundamental gap between \emph{statistical validity} and \emph{decision optimality}: existing approaches either guarantee coverage without considering decision performance, or improve efficiency without explicitly optimizing for the decision task.

\begin{figure}[t]
    \centering
    \includegraphics[width=0.8\linewidth]{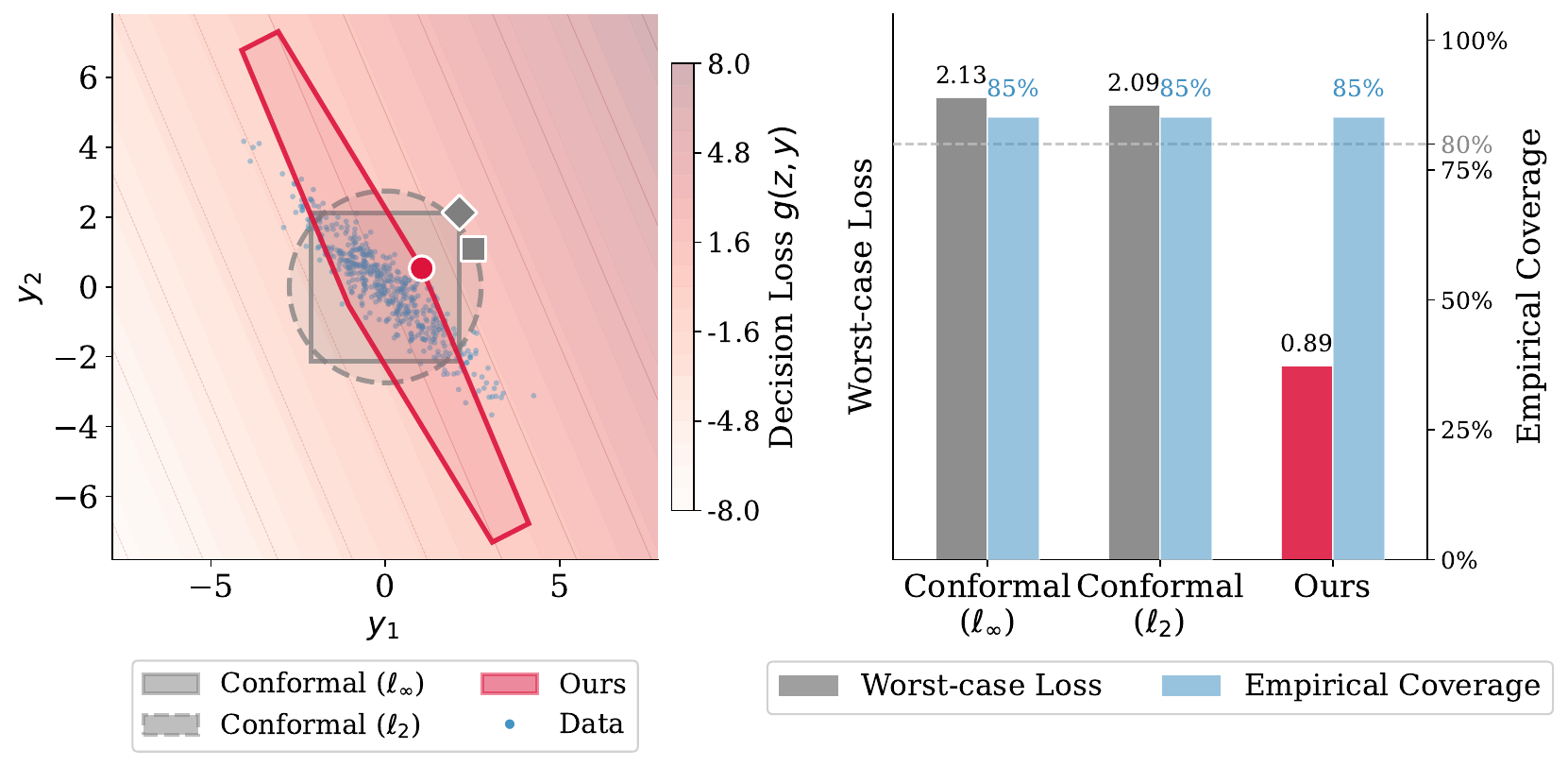}
\caption{Illustration of decision-aware conformal prediction sets. Standard conformal sets (grey) and our decision-aware set (red), overlaid with the loss $g(z,y) = z^\top y$. All three sets achieve the same empirical coverage (85\%), but our set is oriented to avoid high-loss directions, reducing worst-case decision loss. }
    \label{fig:motivating-exp}

\end{figure}

In this paper, we bridge this gap with a decision-aware conformal prediction framework for robust optimization, in which the geometry of the uncertainty set is explicitly aligned with the downstream decision problem rather than treated as a task-agnostic byproduct of prediction (Fig.~\ref{fig:motivating-exp}). We introduce a learnable nonconformity score that induces \emph{polyhedral uncertainty sets} via data-driven hyperplanes, capturing the \emph{directional and anisotropic uncertainty} most relevant to the decision objective while keeping the resulting robust optimization problem tractable. Our method learns the set geometry by minimizing the induced robust loss and maintains statistical validity through conformal calibration, with a recalibration step on an independent dataset that corrects for selection bias. We establish finite-sample coverage for the final, data-dependent set, and derive finite-sample bounds on the suboptimality gap relative to an oracle decision, decomposed into calibration and learning errors, which provide guidance on allocating data between learning and recalibration.

\paragraph{Related work.}
A growing literature adapts conformal prediction (CP) to construct uncertainty sets for robust optimization (RO), using shapes such as ellipsoids~\citep{johnstone2021conformal} and unions of balls~\citep{patel2024conformal}, applying CP to RO problems in power systems~\citep{chen2025enhancing, chen2026large} and autonomous driving~\citep{lekeufack2024conformal}, and characterizing CP sets under RO from decision-theoretic perspectives~\citep{kiyani2025decision, chenreddy2024end}. A parallel line seeks more efficient CP sets via high-density regions~\citep{lei2014distribution, izbicki2022cd, zheng2024generative}, expressive multidimensional scores~\citep{xu2024conformal, tumu2024multi, fang2025contra}, model selection~\citep{liang2024conformal, yang2025selection}, or end-to-end minimization of set size~\citep{bai2022efficient, braun2025minimum, kiyani2024length}. Closest to our work are decision-focused approaches~\citep{wilder2019melding, mandi2024decision} that tailor CP sets to downstream decisions through utility-aware scores~\citep{cortes2024utility}, end-to-end training of neural convex sets~\citep{yeh2024end}, or selection among a finite set of conformal predictors~\citep{bao2025optimal}. In contrast, we directly learn the geometry of a polyhedral conformal set for a general decision objective and restore validity by recalibration; unlike conformal predictive programming~\citep{zhao2024conformal}, which seeks a decision satisfying a conformalized chance constraint, we learn the set geometry itself. Appendix~\ref{app:related} provides a detailed discussion.

\section{Problem Setup and Preliminaries}

In many high-stakes decision-making problems, the true outcome $y \in \mathcal{Y} \subseteq \mathbb{R}^d$ is unknown at the time a decision $z \in \mathcal{Z} \subseteq \mathbb{R}^p$ must be made, where $\mathcal{Z}$ denotes the feasible decision region. A standard approach is to first construct an uncertainty (or prediction) set $\mathcal{C}(x) \subseteq \mathcal{Y}$ based on observed features $x \in \mathcal{X}$, and then select a decision that is robust against all realizations within this set~\citep{johnstone2021conformal}. Formally, given a loss function $g(z,u)$ that evaluates the performance of decision $z$ under a potential outcome $u\in \mathcal{Y}$ in context $x$, the decision is obtained by solving the contextual robust optimization problem
\[
    \min_{z \in \mathcal{Z}} ~ \max_{u \in \mathcal{C}(x)} ~ g(z, u),
\]
where the inner maximization captures adversarial uncertainty over the set $\mathcal{C}(x)$, and the outer minimization selects the decision that minimizes the worst-case loss.

To balance the \emph{trade-off between robustness and informativeness}~\citep{zhou2025robustness}, we seek uncertainty sets that are both \emph{statistically valid} and \emph{decision-aware}. Specifically, we consider a parameterized family of set-valued mappings $\{\mathcal{C}(\cdot;\theta): \mathcal{X} \to 2^{\mathcal{Y}}\}$ indexed by $\theta \in \Theta$. Let $(X,Y) \sim \mathcal{P}$ denote the joint distribution of features and outcomes. Our goal is to learn a parameter $\theta$ such that: ($i$) the resulting uncertainty set achieves a desired coverage level, \ie, $\mathbb{P}\bigl\{Y \in \mathcal{C}(X;\theta)\bigr\} \ge 1 - \alpha$, and ($ii$) the induced robust decision incurs low worst-case loss, which leads to the following constrained learning problem:
\begin{align}
    \min_{\theta \in \Theta} \quad & \ell(\theta) \coloneqq \min_{z \in \mathcal{Z}} ~ \max_{u \in \mathcal{C}(x;\theta)} ~ g(z, u) \label{eq:main_obj} \\
    \text{s.t.} \quad
    & \mathbb{P}\bigl\{ Y \in \mathcal{C}(X;\theta) \bigr\} \ge 1 - \alpha. \label{eq:coverage}
\end{align}
For a fixed context $x$ and parameter $\theta$, the inner min--max problem is a deterministic robust optimization problem that defines the decision induced by $\mathcal{C}(x;\theta)$; the outer optimization then selects, among all sets satisfying~\eqref{eq:coverage}, the one with the best downstream decision performance. Thus,~\eqref{eq:main_obj} is a decision-focused formulation that chooses the geometry of the uncertainty set to directly optimize decision quality.

\subsection{Preliminaries: Conformal Prediction}

Conformal prediction provides a general, model-agnostic framework for constructing prediction sets with finite-sample, distribution-free coverage guarantees \citep{vovk2005algorithmic, shafer2008tutorial, angelopoulos2023conformal}. Suppose the data $\{(X_i, Y_i)\}_{i=1}^{n+1}$ are \emph{exchangeable}. Given calibration data $\mathcal{D}_n \coloneqq \{(X_i, Y_i)\}_{i=1}^n$, the goal is to construct a set-valued predictor $\mathcal{C}(x) \subseteq \mathcal{Y}$ such that, for a new test point $(X_{n+1}, Y_{n+1})$, the coverage guarantee
\begin{equation}
    \mathbb{P}\bigl\{ Y_{n+1} \in \mathcal{C}(X_{n+1};\mathcal{D}_n) \bigr\} \ge 1 - \alpha
    \label{eq:cf-coverage}
\end{equation}
holds for a prescribed miscoverage level $\alpha \in (0,1)$.
\begin{assumption}[Exchangeability]
The augmented sample $\{(X_i, Y_i)\}_{i=1}^{n+1}$ is exchangeable, \ie, its joint distribution is invariant under any permutation of $\{1,\dots,n+1\}$.
\end{assumption}

A common implementation is split conformal prediction~\citep{vovk2005algorithmic, lei2018distribution}. The data are partitioned into a training set and a calibration set. A predictive model $\hat{f}:\mathcal{X}\to\mathcal{Y}$ is first trained on the training set, and a nonconformity score function $s(x,y)$ is defined to measure how well a candidate outcome $y$ conforms to the prediction at $\hat{f}(x)$. Given the calibration data $\mathcal{D}_n$,  the empirical quantile of the nonconformity scores is defined as
\begin{equation}
    R(\mathcal{D}_n) = \inf \left\{ r : \frac{1}{n} \sum_{i=1}^n \mathbbm{1}\bigl\{ s(X_i, Y_i) \le r \bigr\} \ge \frac{\left\lceil (n+1)(1-\alpha) \right\rceil}{n} \right\}.
    \label{eq:quantile}
\end{equation}
The prediction set for a new input $x$ is then constructed as
\[
    \mathcal{C}(x; \mathcal{D}_n) = \{ u \in \mathcal{Y} : s(x,u) \le R(\mathcal{D}_n) \}.
\]
By construction, the set $\mathcal{C}$ satisfies the marginal coverage guarantee in~\eqref{eq:cf-coverage} under the exchangeability assumption, regardless of the choice of the underlying predictive model $\hat{f}$.

\section{Decision-Aware Conformal Calibration}

We now present our decision-aware conformal framework (Fig.~\ref{fig:theory}). It parameterizes polyhedral sets via data-driven hyperplanes, learns their geometry by minimizing the induced robust decision loss, and ensures statistical validity through conformal calibration, restored after this data-dependent selection by recalibrating on an independent dataset.

\begin{figure}[!t]
    \centering
    \begin{subfigure}[t]{0.42\linewidth}
        \centering
        \includegraphics[width=\linewidth]{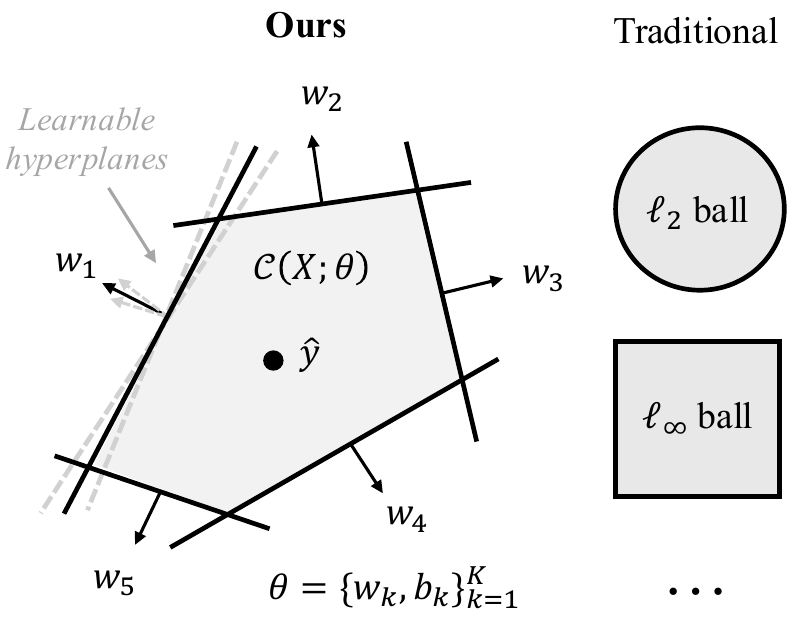}
\caption{The proposed learnable conformal set}
        \label{fig:theory-left}
    \end{subfigure}
    \hspace{0.02\linewidth}
    \begin{subfigure}[t]{0.54\linewidth}
        \centering
        \includegraphics[width=\linewidth]{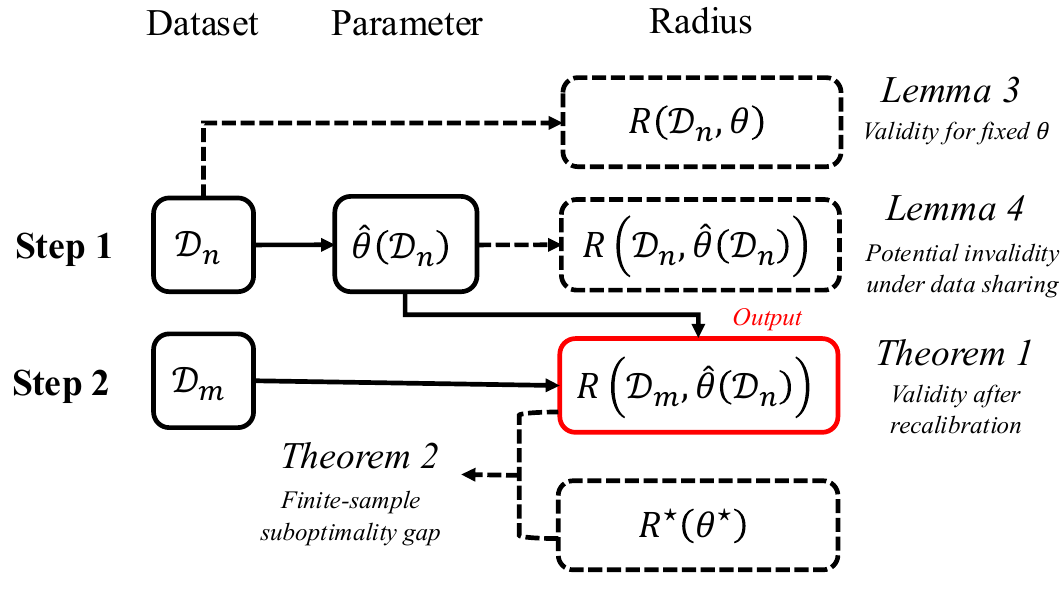}
\caption{Overview of the proposed algorithm}
        \label{fig:theory-right}
    \end{subfigure}
\caption{ Overview of decision-aware conformal set learning. }
    \label{fig:theory}
\end{figure}

\subsection{Learnable Polyhedral Conformal Set}

Unlike classical uncertainty sets with fixed and isotropic geometry (\eg, $\ell_2$ balls)~\citep{ben1998robust, ben1999robust, shafer2008tutorial}, our core idea is to construct a parameterized family of polyhedral sets defined as intersections of data-driven halfspaces (supporting hyperplanes), allowing the geometry to adapt to the downstream decision objective (Fig.~\ref{fig:theory-left}). We require statistical validity to hold uniformly over all hyperplane configurations, and then optimize the parameters within this valid family to minimize downstream decision loss.

Let $\epsilon(x,y) \coloneqq y - \hat{f}(x)$ denote the residual, and define its augmented form $\tilde \epsilon(x,y) \coloneqq [\epsilon(x,y)^\top, 1]^\top \in \mathbb{R}^{d+1}$. We parameterize the nonconformity score as
\begin{equation}
    s(x,y; \theta) \coloneqq \max_{k\in[K]} \left\{\theta_k^\top\tilde \epsilon(x,y)\right\},
    \label{eq:non_conformity}
\end{equation}
where $\theta = \{\theta_k \coloneqq [w_k^\top, b_k]^\top \in \mathbb{R}^{d+1}\}_{k=1}^K$. This score is a max-affine (polyhedral convex) function, which can be interpreted as a polyhedral gauge on the residual space, whose sublevel sets define polyhedral regions shaped by the hyperplanes $\{\theta_k\}_{k=1}^K$; see, \eg,~\citep{niculescu2006convex}. With symmetric choices of $\{\theta_k\}$, it recovers classical polyhedral norms; for example, $\{\pm e_j\}_{j=1}^d$ yields the $\ell_\infty$ norm.

Given $s(x, y;\theta)$, we define a parameter-dependent radius via split conformal calibration on the calibration dataset $\mathcal{D}_n$:
\begin{equation}
\label{eq:radius}
R(\mathcal{D}_n; \theta)
=
\inf \left\{ r \in \mathbb{R}:
\frac{1}{n} \sum_{i=1}^n
\mathbbm{1}\{ s(X_i,Y_i;\theta) \le r \}
\ge
\frac{\left\lceil (n+1)(1-\alpha) \right\rceil}{n}
\right\}.
\end{equation}
For any fixed $\theta$ and a new input $x$, the resulting conformal prediction set
\begin{equation}
    \mathcal{C}(x;\mathcal{D}_n, \theta)
    \coloneqq
    \left\{
        u\in\mathcal{Y}:
        s(x,u;\theta)\le R(\mathcal{D}_n;\theta)
    \right\}
    \label{eq:param-cf-set}
\end{equation}
satisfies the coverage guarantee in~\eqref{eq:cf-coverage}. To ensure that $\mathcal{C}(\cdot;\mathcal{D}_n, \theta)$ is a bounded polyhedron (polytope) for all $x$, we require $K \ge d+1$ and that the collection of normal vectors $\{w_k\}_{k=1}^K \subset \mathbb{R}^d$---corresponding to the first $d$ components of $\theta_k$---positively spans $\mathbb{R}^d$.

\subsection{Proposed Algorithm}
\label{sec:algo}

We now present a two-stage procedure that integrates decision-aware learning of the proposed conformal set with a post-hoc recalibration step that ensures valid coverage (Fig.~\ref{fig:theory-right}).

\paragraph{Step 1: Decision-Aware Conformal Set Learning.}

Given the parameterized conformal set $\mathcal{C}(\cdot;\mathcal{D}_n, \theta)$ in \eqref{eq:param-cf-set}, we learn the parameter $\theta$ by directly minimizing the downstream decision loss induced by the corresponding robust optimization problem:
\begin{align}
    \mathcal{L}_x(\theta,r) \coloneqq \min_{z \in \mathcal{Z}} \quad &  \max_{u \in \mathcal{Y}} ~ g(z,u)
    \label{eq:regret-new}\\
    \text{s.t.} \quad
    & s(x,u; \theta) \le r.
    \label{eq:robust-new}
\end{align}
We aim to minimize its empirical average over the learning dataset $\mathcal{D}_n$,
$$
\hat{\theta}(\mathcal{D}_n)
\coloneqq
\arg\min_{\theta\in\Theta}
\frac{1}{n}\sum_{i=1}^n
\mathcal{L}_{X_i}\!\left(\theta,R(\mathcal{D}_n;\theta)\right).
$$
Here, we explicitly denote the dependence of $\hat{\theta}$ on the learning data $\mathcal{D}_n$. In the following sections, we suppress the subscript $x$ from $\mathcal{L}_x$ when clear from context.

Compared to the original formulation~\eqref{eq:main_obj}--\eqref{eq:coverage}, the coverage constraint is enforced implicitly through conformal calibration via $R(\mathcal{D}_n;\theta)$. This allows the probabilistic constraint to be replaced by the deterministic constraint~\eqref{eq:robust-new}, thereby decoupling statistical validity from the optimization procedure. Moreover, the polyhedral structure induced by $s(\cdot, \cdot;\theta)$ implies that the feasible set in~\eqref{eq:robust-new} can be represented as a finite collection of linear inequalities.

\begin{remark}[Connection with sample average approximation]
The radius condition $r\ge R(\mathcal{D}_n;\theta)$ is equivalent to
$$\frac{1}{n} \sum_{i=1}^n
\mathbbm{1}\{ s(X_i,Y_i;\theta) \le r \}
\ge
 1-\omega_n
, \quad
\text{where }
\omega_n = 1-\frac{\left\lceil (n+1)(1-\alpha) \right\rceil}{n}.
$$
\end{remark}
Therefore, the radius condition can be viewed as an $n$-dependent sample average approximation (SAA) constraint with the conformal choice of $\omega_n$~\citep{luedtke2008sample, zhao2024conformal}.

\paragraph{Step 2: Recalibration for Robust Decision-Making.}

The conformal calibration in~\eqref{eq:radius} guarantees valid coverage for any \emph{fixed} parameter $\theta$. However, in Step~1, $\hat\theta$ is learned from the data $\mathcal{D}_n$, which introduces a data-dependent selection effect~\citep{hegazy2025validselectionconformalsets, liang2024conformal} and violates the exchangeability condition required for conformal validity.

To restore validity, we perform a recalibration step using an independent dataset. Specifically, let $\mathcal{D}_m \coloneqq \{(X_j, Y_j)\}_{j=1}^m$ denote an additional set of exchangeable samples drawn from the same distribution. Given the learned parameter $\hat{\theta}(\mathcal{D}_n)$, we compute the conformal radius on this fresh calibration set, yielding $R(\mathcal{D}_m; \hat{\theta})$. The resulting conformal set $\mathcal{C}(x; \mathcal{D}_m, \hat{\theta})$ is then used to prescribe the final decision
\begin{equation}
\label{eq:final-decision}
\hat{z}(x)
=
\arg \min_{z \in \mathcal{Z}} ~ \max_{u}~g(z,u),\quad\text{s.t.}~u \in \mathcal{C}(x; \mathcal{D}_m, \hat{\theta}(\mathcal{D}_n)).
\end{equation}
This recalibration step is essential to ensure valid post-selection coverage. Conformal guarantees rely on exchangeability conditional on a fixed nonconformity score; however, optimizing $\theta$ using data breaks this condition. Intuitively, reusing the same data for both tuning $\theta$ and calibrating the radius leads to overfitting of the score function and overly optimistic (\ie, too small) uncertainty sets, resulting in undercoverage. By recalibrating on an independent dataset, we restore the required independence and recover valid coverage guarantees for the final decision.

\subsection{Computational Methods}
\label{sec:comp-methods}

To operationalize the proposed algorithm, we derive finite-dimensional reformulations of the decision-aware conformal set learning problem in~\eqref{eq:regret-new}--\eqref{eq:robust-new} for linear and more general objective functions $g$. Two main computational challenges arise: ($i$) the nonconformity score in~\eqref{eq:non_conformity} introduces a nonsmooth max-affine constraint, and ($ii$) the conformal radius $R(\mathcal D_n;\theta)$ is an empirical quantile that depends on $\theta$ through a nonsmooth order statistic. All proofs are included in Appendix~\ref{app:proof}.

We begin by reformulating the max-affine score constraint in~\eqref{eq:robust-new} as a linear system for fixed conformal radius $r$. Recall that $\theta_k = [w_k^\top, b_k]^\top$, where $w_k \in \mathbb{R}^d$ and $b_k \in \mathbb{R}$. Define
\[
W \coloneqq [w_1,\ldots,w_K]^\top \in \mathbb{R}^{K \times d},
\qquad
b \coloneqq [b_1,\ldots,b_K]^\top \in \mathbb{R}^K,
\qquad
c_n \coloneqq \left\lceil (n+1)(1-\alpha)\right\rceil.
\]
\begin{lemma}[Linear reformulation of the score constraint]
\label{lem:score-linear}
For any radius $r$, the constraint $s(x,y;W,b) \le r$ is equivalent to the following system of linear inequalities
\[
    W \epsilon(x,y) + b \le r 1_K,
\]
where $1_K$ denotes the all-ones vector in $\mathbb{R}^K$.
\end{lemma}

To capture the dependence of the conformal radius $R(\mathcal D_n;W,b)$ on $(W,b)$ using the big-$M$ method~\citep{nemhauser1988integer, williams2013model}, we introduce an auxiliary scalar $r\in\mathbb R$ and binary variables $t=[t_1,\ldots,t_n]^\top\in\{0,1\}^n$.
\begin{lemma}[Big-$M$ representation of the conformal radius]
\label{lem:radius-bigM}
Suppose $M>0$ is such that, for all feasible $(W,b)$ and $i\in[n]$, $R(\mathcal D_n;W,b)\ge\underline r$ and $W\epsilon_i+b \le (\underline r+M)1_K$. Then the empirical conformal radius can be represented as
\begin{equation}
\label{eq:radius-bigM}
\begin{aligned}
R(\mathcal D_n;W, b)
=
\min_{r, t}\quad & r \\
\text{s.t.}\quad
& W \epsilon_i+b
\le
(r+ M(1-t_i)) 1_K,
&& i\in[n],
\\
& 1_n^\top t \ge c_n, \quad t\in\{0,1\}^n.
\end{aligned}
\end{equation}
Here, $t_i=1$ enforces that the $i$th nonconformity score is no larger than $r$, whereas $t_i=0$ relaxes the corresponding constraint. The constraint $1_n^\top t\ge c_n$ therefore requires at least $c_n$ nonconformity scores to be bounded by $r$, and minimizing $r$ recovers the empirical conformal quantile.
\end{lemma}

Combining Lemmas~\ref{lem:score-linear} and~\ref{lem:radius-bigM} gives the following exact reformulations to problem~\eqref{eq:regret-new}--\eqref{eq:robust-new}.

\begin{prop}[Reformulation of problem~\eqref{eq:regret-new}--\eqref{eq:robust-new}]
\label{prop:efficient-solution}
The problem~\eqref{eq:regret-new}--\eqref{eq:robust-new} is equivalent to
\begin{equation}
\label{eq:finite-conformal-robust}
\begin{aligned}
\min_{W,b,z,r,t}
\quad & \max_{u\in\mathcal Y}~g(z,u) \\
\text{s.t.}\quad
& W \left(u - \hat{f}(x)\right)+b \le r1_K,\\
& W \epsilon_i+b \le (r+ M(1-t_i)) 1_K, \qquad i\in[n], \\
& 1_n^\top t\ge c_n, \quad z\in\mathcal Z, \quad t\in\{0,1\}^n.
\end{aligned}
\end{equation}
\end{prop}

\begin{corollary}[Reformulation for linear objective]
\label{corollary:reform-linear}
Suppose $\mathcal Y=\mathbb R^d$ and $g(z,y)=g_0(z)+z^\top Qy$. The min--max problem in~\eqref{eq:finite-conformal-robust} admits the following equivalent reformulation.
\begin{equation}
\label{eq:linear-exact-reformulation}
\begin{aligned}
\min_{W,b,z,\lambda,r,t}\quad & g_0(z) + \left(r1_K-b+W\hat f(x)\right)^\top \lambda
\\
\text{s.t.}\quad
& W^\top\lambda = Q^\top z,\qquad \lambda\in\mathbb R_+^K,\\
& W\epsilon_i+b \le (r+M(1-t_i))1_K, \qquad i\in[n],\\
& 1_n^\top t\ge c_n, \quad z\in\mathcal Z,\quad t\in\{0,1\}^n .
\end{aligned}
\end{equation}
\end{corollary}

The mixed-integer formulation in~\eqref{eq:linear-exact-reformulation} can be solved exactly using spatial branch-and-bound solvers such as Gurobi~\citep{gurobi} for small $K$ and $n$. For general concave objectives $g(z,\cdot)$ in~\eqref{eq:finite-conformal-robust}, we employ Column-and-Constraint Generation (CCG)~\citep{zhao2012exact}, which alternates between a master problem over $(W,b,z,r,t)$ and an adversarial subproblem over $u$ (Alg.~\ref{alg:ccg} in Appendix~\ref{app:pinball-quantile}). However, the master problem remains a mixed-integer program with exponential complexity in $n$~\citep{nemhauser1988integer}, and when $g(z,\cdot)$ is convex, the adversarial subproblem requires searching over the vertices of the polytope, which can number $\mathcal{O}(K^{\lfloor d/2\rfloor})$ in $\mathbb R^d$.

For a more efficient solution, we relax the empirical quantile constraint in~\eqref{eq:radius-bigM} using a continuous pinball-loss penalty~\citep{Steinwart_2011}. This replaces the $n$ binary variables with a piecewise-linear penalty and enables stochastic subgradient methods on minibatches of calibration samples, substantially reducing the per-iteration complexity (Alg.~\ref{alg:pinball} and Appendix~\ref{app:pinball-quantile}).

\section{Theoretical Analysis}
\label{sec:theory}

This section presents two central theoretical properties: ($i$) the recalibrated parameter yields a valid conformal prediction set, and ($ii$) the suboptimality gap admits a finite-sample characterization under mild regularity conditions. The proofs are deferred to Appendix~\ref{app:proofs-theory}.

\subsection{Statistical Validity of the Recalibrated Conformal Set}

We begin by recalling the standard validity guarantee of conformal prediction for a fixed nonconformity score, which is well established in the literature~\citep{angelopoulos2023conformal}.
\begin{lemma}[Validity for Fixed $\theta$]
    \label{thm:validity}
Under the exchangeability assumption, for any fixed $\theta \in \Theta$, the conformal prediction set $\mathcal{C}(\cdot;\mathcal{D}_n,\theta)$ defined in~\eqref{eq:radius}--\eqref{eq:param-cf-set} satisfies
    \[
    \mathbb{P}\bigl\{ Y_{n+1} \in \mathcal{C}(X_{n+1}; \mathcal{D}_n, \theta) \bigr\} \ge 1 - \alpha.
    \]
\end{lemma}

When $\theta$ is instead fitted to the calibration data, yielding $\hat{\theta}(\mathcal{D}_n)$, the augmented scores are no longer exchangeable, and the prediction set tends to overfit the calibration samples, leading to undercoverage. The following result formalizes this breakdown; its proof follows arguments similar to Theorem~2 by~\cite{barber2023conformal} and Corollary~1 by~\cite{zhou2025conformalized}. Let $d_{\mathrm{TV}}(\mathcal{D}_{n+1}, \mathcal{D}_{n+1}^{(i)})$ denote the total variation distance between the nonconformity scores of the augmented sample $\mathcal{D}_{n+1} = \{(X_1,Y_1),\ldots,(X_{n+1},Y_{n+1})\}$ computed with $\hat\theta$ fitted on the first $n$ entries of $\mathcal{D}_{n+1}$ and on those of the sequence $\mathcal{D}_{n+1}^{(i)}$ obtained by exchanging its $(n+1)$-th and $i$-th entries. This quantity measures how much exchangeability is violated: it is zero when $\hat\theta$ does not depend on $\mathcal{D}_n$ and grows with the dependence between $\hat\theta(\mathcal{D}_n)$ and the nonconformity scores.

\begin{lemma}[Potential Invalidity Under Data Sharing]
    \label{thm:invalid}
Under the exchangeability assumption, if the learning data are reused in constructing the conformal prediction set $\mathcal{C}(\cdot;\mathcal{D}_n,\hat\theta(\mathcal{D}_n))$, then:
    \[
    \mathbb{P}\bigl\{ Y_{n+1} \in \mathcal{C}(X_{n+1}; \mathcal{D}_n, \hat \theta(\mathcal{D}_n))\bigr\} \ge 1 - \alpha - \tfrac{1}{n+1}\sum\nolimits_{i=1}^{n} d_{\rm TV}\left(\mathcal{D}_{n+1}, \mathcal{D}_{n+1}^{(i)}\right).
    \]
\end{lemma}

Our algorithm prevents such invalidity via a recalibration step using an independent split dataset, thereby restoring finite-sample validity even when $\theta$ is data-dependent.
\begin{theorem}[Validity of the Recalibrated Set]
\label{thm:re-calibrated-validity}
Suppose the recalibration data $\{(X_j,Y_j)\}_{j=1}^m$, the test point $(X_{m+1},Y_{m+1})$, and the data used to obtain $\hat{\theta}$ are mutually exchangeable. Then the recalibrated prediction set satisfies
\[
\mathbb{P}\bigl\{ Y_{m+1} \in \mathcal{C}(X_{m+1}; \mathcal{D}_m, \hat{\theta}(\mathcal{D}_n)) \bigr\}
\ge 1-\alpha .
\]
\end{theorem}

\subsection{Suboptimality Gap of the Prescribed Decision}
We now quantify the gap between the decision induced by the learned parameter $\hat{\theta}$ and that of an oracle choice $\theta^\star$, establishing finite-sample bounds on how calibration, learning, and optimization errors propagate to the final decision.

To formalize this analysis, we first introduce oracle benchmarks and a radius-explicit formulation of the robust loss. For a fixed $x$ and any $\theta$, define the oracle conformal radius
\[
R^\star(\theta) = \inf_{r}\left\{ r \in \mathbb{R} : \mathbb{P}\left\{ s(X, Y;\theta) \leq r \right\} \geq 1-\alpha \right\},
\]
and the corresponding oracle parameter
\begin{equation}
    \label{eq:oracle}
    \theta^\star  =\arg\min_{\theta} \min_{z \in \mathcal{Z}} ~ \max_{u}~g(z,u),\quad\text{s.t.}~
    s(x,u;\theta) \le R^\star(\theta).
\end{equation}
Recall that the robust decision loss defined in \eqref{eq:regret-new} is denoted by $\mathcal{L}(\theta, r)$, where we suppress the subscript $x$.

We impose the following regularity conditions, which are standard in empirical process analysis and control the estimation of the conformal radius and the stability of the induced robust objective.

\begin{assumption}[Regularity Conditions]
    \label{ass:gap}
Assume all data are independent, and:
    \begin{enumerate}[leftmargin=*, itemsep=0pt, topsep=0pt, parsep=0pt, partopsep=0pt]
\item[$(a)$] The function class $\left\{ (x, y) \mapsto \mathbbm{1}\left\{ s(x, y;\theta) \leq r \right\}, \theta \in \Theta, r \in \mathbb{R} \right\}$ has VC dimension $V$. \item[$(b)$] There exists $\kappa > 0$ such that for all $\theta \in \Theta$, $ f_\theta(t) \ge \kappa$ for $t$ in a neighborhood of $R^\star(\theta)$, where $f_\theta$ is the density of the nonconformity scores $s(X, Y; \theta)$. \item[$(c)$] There exists $L_r \geq 0$ such that for all $\theta \in \Theta$ and $x \in \mathcal{X}$, there is:
        \[
        \left|\mathcal{L}(\theta, r) - \mathcal{L}(\theta, r') \right| \leq L_r \cdot |r - r'|
        \]
\item[$(d)$] Problem~\eqref{eq:regret-new}--\eqref{eq:robust-new} is solved to optimization tolerance $\epsilon_{\rm opt}$.
    \end{enumerate}
\end{assumption}
Condition $(a)$ ensures uniform convergence over the score class, $(b)$ guarantees well-behaved quantile estimation near the target radius, $(c)$ ensures stability of the downstream objective with respect to the radius, and $(d)$ accounts for optimization error.

\begin{theorem}[Finite-Sample Suboptimality Gap]
    \label{thm:gap}
Let $\hat \theta$ be the estimated parameter obtained by solving~\eqref{eq:regret-new}--\eqref{eq:robust-new}, and let $\theta^\star$ be the oracle parameter obtained by solving \eqref{eq:oracle}. Let $\mathcal{L}(\cdot)$ be the robust loss defined in \eqref{eq:regret-new}, and denote the suboptimality gap as: $ \Delta \coloneqq \mathcal{L}( \hat\theta(\mathcal{D}_n), R(\mathcal{D}_m;\hat \theta(\mathcal{D}_n))) - \mathcal{L}\left( \theta^\star, R^\star(\theta^\star)\right). $ Then, under Assumption~\ref{ass:gap} and for sufficiently large $m,n$, with probability at least $1-\delta$:
    \[
    \Delta \leq
    \frac{L_r}{\kappa} \sqrt{\frac{\log\left(4 / \delta \right)}{2m}} +
    \frac{2 L_r}{\kappa}
    \sqrt{
    \frac{
    32\left[
    V\log(en/V)+\log(16/\delta)
    \right]
    }{n}
    } + \frac{2L_r}{\kappa}\Bigl(\frac{1}{m}+\frac{2}{n}\Bigr) + \epsilon_{\rm opt}.
    \]
\end{theorem}
In $\mathcal{O}_{\mathbb{P}}$ notation~\citep{van2000asymptotic}, the bound simplifies to
\[
\Delta \leq \mathcal{O}_{\mathbb{P}}\left( \sqrt{\frac{1}{m}} + \sqrt{\frac{\log n}{n}} \right) + \epsilon_{\rm opt}.
\]
In addition to the optimization tolerance, the bound contains a calibration error scaling as $m^{-1/2}$ and a learning error scaling as $\sqrt{(\log n)/n}$. Since the calibration error decays faster, it is typically beneficial to allocate more data to learning (Step~1) than to recalibration (Step~2). Because the bound involves unknown constants ($V$, $L_r$, $\epsilon_{\rm opt}$, and $\kappa$), we use it to guide the data-splitting ratio rather than as a quantitative measure of decision suboptimality. The proof is provided in Appendix~\ref{app:gap}.

\begin{remark}
Theorems~\ref{thm:re-calibrated-validity} and~\ref{thm:gap} continue to hold when the pinball surrogate is used: validity relies only on the exchangeability of the recalibration data, and the error introduced by the surrogate is absorbed into $\epsilon_{\rm opt}$.
\end{remark}

\section{Experiments}

\begin{figure}[!t]
    \centering
    \includegraphics[width=\linewidth]{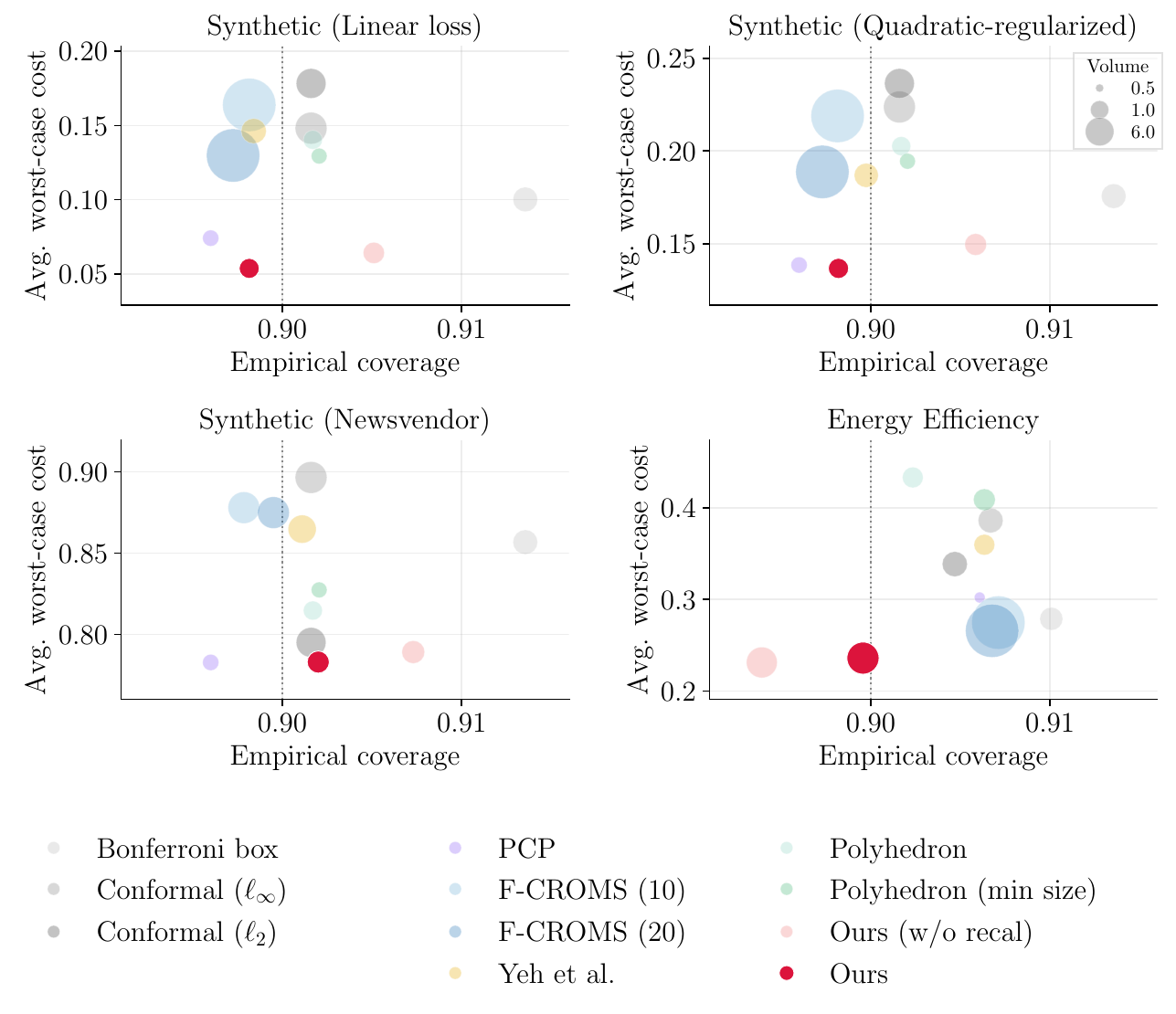}
\caption{The proposed method in comparison with baselines. The horizontal axis shows the empirical coverage rate; the vertical axis presents the average worst-case cost; and bubble size represents the volume of the uncertainty set. }
    \label{fig:results}

\end{figure}

We evaluate the proposed decision-aware conformal framework on robust decision problems using both synthetic and real-data tasks. Across all experiments, we first fit a predictor $\hat f:\mathcal X\to\mathbb R^d$ on an independent training split and construct residuals $\epsilon_i=Y_i-\hat f(X_i)$. The remaining data are split into a learning set $\mathcal D_n$, used to learn the set geometry $\hat\theta(\mathcal D_n)$, and an independent calibration set $\mathcal D_m$, used only for the final conformal radius. This yields the final set $\mathcal C(x;\mathcal D_m,\hat\theta(\mathcal D_n))$, from which robust decisions are computed on held-out test covariates. Baselines that do not learn set geometry use $\mathcal D_n\cup\mathcal D_m$ as their calibration set.

We compare against decision-agnostic conformal sets, including Bonferroni boxes~\citep{dunn1961multiple}, conformal boxes or balls~\citep{vovk2005algorithmic}, and PCP~\citep{pmlr-v206-wang23n}, as well as decision-aware alternatives based on F-CROMS~\citep{bao2025optimal} and PICNN scores~\citep{yeh2024end}. We also include three ablations of our method: a polyhedral set without decision-aware learning, a size-aware polyhedron that uses volume as the loss function, and our decision-aware polyhedron without the post-selection recalibration step. For our method, each geometry-learning problem is first solved exactly as the MIQCP of Corollary~\ref{corollary:reform-linear} with Gurobi under a 30-minute limit; if Gurobi does not certify optimality within this limit, or the reformulation does not apply, we use the pinball-loss gradient method of Section~\ref{sec:comp-methods} (Appendix~\ref{app:ours-impl}). We evaluate empirical coverage, average worst-case cost, average regret, set volume, and wall-clock time. Additional details on the objective function, datasets, and implementation are deferred to Appendix~\ref{app:experiments}. Results are averaged over 20 random runs, except that we use 80 runs for the real-data Energy Efficiency task and five runs for the sensitivity analysis.

\begin{figure}[!t]
    \centering
    \includegraphics[width=\linewidth]{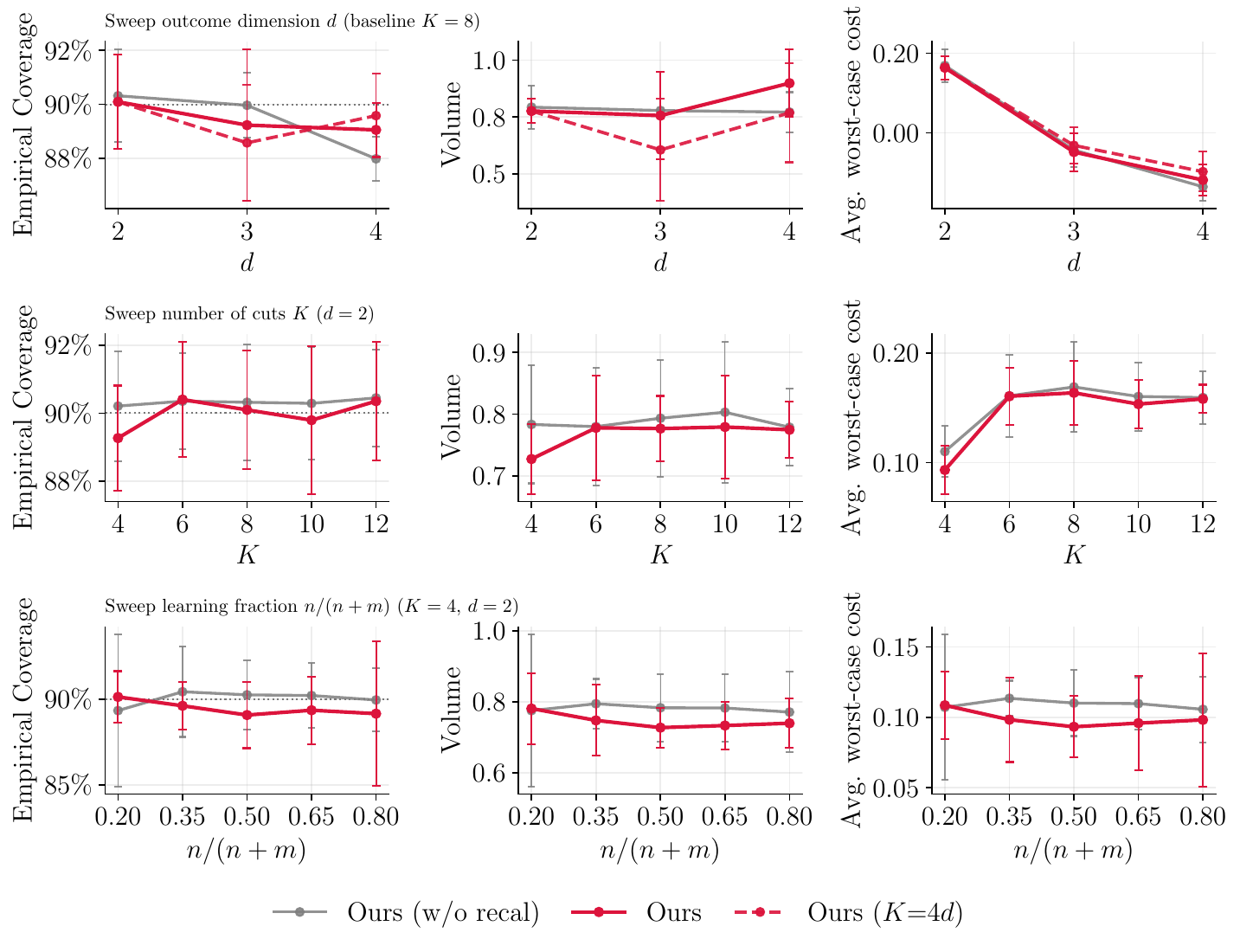}
\caption{Hyperparameter sensitivity of coverage, volume, and worst-case cost. Ours uses $K=8$ in the first row and $K=4$ in the last row. }
    \label{fig:sensitivity}

\end{figure}

\paragraph{Main results.}
Fig.~\ref{fig:results} and Tables~\ref{tab:combined_main_results} and~\ref{tab:energy_efficiency_linear_auto} summarize the results. Within the 30-minute limit, Gurobi certified the MIQCP in 50 of the 80 energy-efficiency runs and 3 of the 20 linear-objective runs; all other runs use the pinball-loss gradient method. Across all tasks, our recalibrated decision-aware polyhedron attains empirical coverage close to the nominal $90\%$ level, confirming that the independent calibration split restores validity after learning the set geometry. Among valid methods, our method consistently achieves competitive or lowest worst-case cost. The comparison with PCP~\citep{pmlr-v206-wang23n} and the size-aware polyhedron shows that sharper or smaller uncertainty sets are not necessarily better for downstream robust optimization: although both often produce smaller-volume sets, this does not consistently translate into lower worst-case cost or regret, as minimizing volume can remove directions that are important for lowering the worst-case decision loss. Compared with F-CROMS~\citep{bao2025optimal}, adding more candidate sets of hyperplanes can reduce regret, but its model-selection approach remains less flexible than ours because it selects among a finite candidate class rather than directly learning the conformal-set geometry. The ablation without recalibration highlights the importance of post-selection calibration. Although its worst-case cost is often close to that of the recalibrated method, its coverage can fall below the target level (Fig.~\ref{fig:sensitivity}, $d=4$) because the same data are used both to select the set geometry and to calibrate the radius, consistent with Lemma~\ref{thm:invalid}. By contrast, the proposed two-stage procedure restores the coverage guarantee with little loss in downstream performance.

\paragraph{Sensitivity analysis.}
Fig.~\ref{fig:sensitivity} studies the sensitivity of our method to $d$, $K$, and $n/(n+m)$. The MIQCP, which applies only at $K=4$ and $d=2$, was certified in two of five runs at $n/(n+m)=0.2$; all other points use the pinball-loss gradient method. As $d$ increases, coverage remains close to the nominal $90\%$ level, while volume grows to satisfy the joint coverage. Scaling the number of hyperplanes with $K=4d$ helps control the volume growth at a similar worst-case cost, suggesting that the polyhedral template should be more expressive in higher dimensions. Varying $K$ has only a mild effect on coverage. Larger $K$ initially increases volume and worst-case cost, which then remain largely stable. Finally, varying $n/(n+m)$ shows the expected learning--calibration trade-off: allocating more data to $\mathcal D_n$ improves geometry learning and can reduce cost, consistent with Theorem~\ref{thm:gap}. But leaving too few samples in $\mathcal D_m$ also makes the conformal radius less stable and can lead to coverage below the nominal level.

\section{Conclusion and Discussion}

We introduced a decision-aware conformal framework for robust optimization that learns uncertainty sets by optimizing downstream robust loss while preserving finite-sample coverage through independent recalibration. The polyhedral parameterization captures directional uncertainty while admitting finite-dimensional reformulations. We established validity for the final learned set and derived a suboptimality-gap bound that separates calibration and learning errors, showing that conformal sets can be designed for both coverage and decision quality. A limitation is the focus on polyhedral geometry, which may be less suitable for highly nonlinear uncertainty. Future work could extend this framework to richer neural or flow-based set families while maintaining calibration validity and computational tractability.

\section*{Acknowledgments}

The project was supported by the National Science Foundation (Grant No.~CAIG-2425888). We acknowledge partial support from the 2024 Block Center Seed Fund at Carnegie Mellon University and the 2024 GenAI Fellows program offered by the Tepper School of Business.

\bibliographystyle{plainnat}
\bibliography{ref}

\appendix

\section{Extended Related Work}
\label{app:related}

Adapting conformal prediction (CP) methods to construct valid uncertainty sets in robust optimization (RO) has been a key focus of recent literature development in CP. Some works propose using different shapes, such as using ellipsoids~\citep{johnstone2021conformal} and multiple generated balls~\citep{patel2024conformal}, to accurately carve the shape of the data and prevent overconservatism. Other works find applications of CP in various real-world RO problems, such as power system resilience~\citep{chen2025enhancing, chen2026large} and autonomous driving~\citep{lekeufack2024conformal}. Another line of literature studies the properties of CP methods in different types of RO problems, such as linear-quadratic regulator systems~\citep{patel2025conformal}, human decision making~\citep{hullman2025conformal}, inverse optimization problems~\citep{lin2024conformal}, and decision-theoretic perspectives~\citep{kiyani2025decision, chenreddy2024end} to derive optimality conditions for CP methods under RO settings and gain implementation insights. Our work is closely related to this line of research, as we also aim to tailor conformal prediction to better serve robust optimization. However, we differ in our approach: rather than deriving theoretical characterizations of the uncertainty set, we adopt a learning-based framework to estimate its optimal shape.

Tailoring conformal prediction (CP) sets to be more efficient (\ie, tighter) is another important direction of CP research. Previous work has established the optimality of constructing sets in high-density regions (HDR) of the underlying target distribution~\citep{lei2014distribution,izbicki2022cd}, which motivates a series of studies studying the use of generative models as prediction models to generate multiple fine-grained sets whose union approximates the contour of the HDR~\citep{pmlr-v206-wang23n, zheng2024generative}. On the other hand, multidimensional CP focuses on developing more expressive nonconformity scores that allow the prediction set to take flexible shapes, including ellipsoidal sets~\citep{xu2024conformal}, general convex sets~\citep{tumu2024multi}, and nonconvex sets~\citep{fang2025contra}. More recently, several directions of work have explored how conformal prediction set shapes can be explicitly optimized. This can be formulated as a model selection problem~\citep{liang2024conformal, yang2025selection}, where the task is to select the best performing model from a finite collection of conformal prediction sets. Another line of work formulates this as an end-to-end learning task whose objective is to minimize set size subject to the validity chance-constraint, and can be solved using, \eg, empirical risk minimization~\citep{bai2022efficient, braun2025minimum} and duality~\citep{kiyani2024length}. Our work is closely related to this line of research, as we also formulate a learning problem. However, rather than optimizing an efficiency-oriented objective as in prior works, we consider a general decision-making objective that subsumes efficiency as a special case.

Our work is closely related to decision-focused learning (DFL)~\citep{wilder2019melding, mandi2024decision}, a class of approaches that train predictive models by optimizing decision-oriented objectives through differentiable optimization layers~\citep{amos2017optnet, agrawal2019differentiable}, and has been adapted to settings such as generative modeling~\citep{wang2025gen} and sequential modeling~\citep{11297004}. The idea of DFL can be extended to learn conformal prediction models tailored for robust optimization, in a way similar to efficiency-aware conformal prediction~\citep{stutz2022learning}. However, this field remains underexplored. For example,~\citet{cortes2024utility} propose decision-aware nonconformity scores that incorporate downstream utility into conformal prediction, yielding sets aligned with decision objectives while preserving coverage guarantees. \citet{yeh2024end} consider a more general framework by parameterizing convex uncertainty sets via neural scores and training them end-to-end using differentiable optimization layers. \citet{bao2025optimal} instead study a model selection approach, selecting from a discrete set of conformal predictors to minimize downstream decision risk. Our work shares the same motivation as tailoring conformal prediction to downstream decision-making while maintaining validity. However, rather than designing specific scores, relying on end-to-end training, or selecting from a finite model class, we adopt a learning-based approach that directly optimizes the shape of the polyhedral conformal prediction set for a general decision objective to achieve optimality.

Our work is also related to chance-constrained optimization (CCO), which addresses uncertainty by requiring constraints to hold with a prespecified probability, thereby explicitly controlling the risk of constraint violation \citep{charnes1959chance, miller1965chance, jiang2016data}. Since directly evaluating chance constraints is often computationally difficult, classical data-driven approaches replace them with sampled approximations, including scenario programs that enforce all sampled constraints~\citep{calafiore2006scenario} and sample-average approximations (SAA) that permit a controlled fraction of violations~\citep{luedtke2008sample}. For finite-support problems, such empirical chance constraints can also be encoded using mixed-integer formulations~\citep{luedtke2010integer}. \citet{geng2019data} provide a comprehensive review of these methods and their applications to decision-making in power systems. More recently, \citet{zhao2024conformal} introduce conformal predictive programming, which combines an empirical-quantile reformulation with independent conformal calibration to provide finite-sample a posteriori feasibility guarantees. Our formulation shares a similar setting where we utilize conformal prediction within the constraint of the optimization problem, but our core objective differs: rather than seeking a decision that satisfies a conformalized chance constraint, we aim to learn the geometry of a polyhedral prediction set that minimizes downstream robust decision loss while preserving finite-sample coverage validity.

\section{Proofs in Section~\ref{sec:comp-methods}}
\label{app:proof}

\subsection{Proof of Lemma~\ref{lem:score-linear}}

\begin{proof}
By definition of the max-affine nonconformity score,
\[
    s(x,y;\theta)
    =
    \max_{k\in[K]}
    \{w_k^\top \epsilon(x,y)+b_k\}.
\]
Therefore, for any $r\in\mathbb R$,
\[
    s(x,y;\theta)\le r
    \quad\Longleftrightarrow\quad
    w_k^\top \epsilon(x,y)+b_k\le r,
    \quad \forall k\in[K].
\]
Stacking these $K$ inequalities gives
\[
    W\epsilon(x,y)+b\le r1_K.
\]
\end{proof}

\subsection{Proof of Lemma~\ref{lem:radius-bigM}}

\begin{proof}
By definition in~\eqref{eq:radius},
\begin{equation}
\begin{aligned}
R(\mathcal{D}_n; \theta)
&=
\inf \left\{ r \in \mathbb{R}:
\frac{1}{n} \sum_{i=1}^n
\mathbbm{1}\{ s(X_i,Y_i;\theta) \le r \}
\ge
\frac{\left\lceil (n+1)(1-\alpha) \right\rceil}{n}
\right\} \\
&= \min \{ r \in \mathbb{R} : \sum_{i=1}^n \mathbbm{1}\{ s(X_i, Y_i; \theta) \le r \} \ge c_n \}.
\end{aligned}
\end{equation}

By Lemma~\ref{lem:score-linear}, $s(X_i,Y_i;\theta)\le r$ can be expressed as
$$
W \epsilon_i + b \le  r1_K =(r + M(1-t_i))1_K
$$
when $t_i=1$, while for $t_i=0$ the constraint $W\epsilon_i+b\le(r+M)1_K$ holds for all $r\ge\underline r$, in particular for $r=R(\mathcal D_n;W,b)$, by the choice of $M$. Therefore,
\begin{equation}
\begin{aligned}
R(\mathcal D_n;W, b)
=
\min_{r, t}\quad & r \\
\text{s.t.}\quad
& W \epsilon_i+b
\le
(r+ M(1-t_i)) 1_K,
&& i\in[n],
\\
& 1_n^\top t \ge c_n, \quad t\in\{0,1\}^n.
\end{aligned}
\end{equation}
\end{proof}

\subsection{Proof of Proposition~\ref{prop:efficient-solution}}
\label{app:closed-form}

Define an auxiliary variable $r$. The worst-case value
\begin{align}
    \max_{u \in \mathcal{Y}} ~ & g(z,u)\\
    \text{s.t.} \quad
    & s(x,u; \theta) \le r
\end{align}
is nondecreasing in $r$, because increasing $r$ enlarges the feasible region of the inner maximization. Hence, problem~\eqref{eq:regret-new}--\eqref{eq:robust-new} is equivalent to:
\begin{equation}
\label{eq:r_reform}
\min_{\substack{\theta,z,r\\
r\ge R(\mathcal D_n;\theta)}}
\max_{\substack{u\in\mathcal Y\\s(x,u;\theta)\le r}}
g(z,u).
\end{equation}
By Lemma~\ref{lem:score-linear}, the score constraint in the inner maximization layer admits the linear representation
\[
    s(x,u;W,b)\le r
    \quad\Longleftrightarrow\quad
    W(u-\hat f(x))+b\le r1_K .
\]
It remains to represent the constraint $r\ge R(\mathcal D_n;W,b)$ in the outer minimization layer. By Lemma~\ref{lem:radius-bigM}, for a valid choice of $M$:
\[
R(\mathcal D_n;W,b)
=
\min_{r,t}
\left\{
r:
W\epsilon_i+b\le (r+M(1-t_i))1_K,\ i\in[n],
\quad
1_n^\top t\ge c_n,\quad
t\in\{0,1\}^n
\right\}.
\]
Equivalently, $r\ge R(\mathcal D_n;W,b)$ holds if and only if there exists $t\in\{0,1\}^n$ such that
\[
    W\epsilon_i+b\le (r+M(1-t_i))1_K,
    \qquad i\in[n],
    \qquad
    1_n^\top t\ge c_n .
\]

Therefore, \eqref{eq:r_reform} is equivalent to:
\begin{equation}
\begin{aligned}
\min_{W,b,z,r,t}
\quad & \max_{u\in\mathcal Y}~g(z,u) \\
\text{s.t.}\quad
& W \left(u - \hat{f}(x)\right)+b \le r1_K,\\
& W \epsilon_i+b \le (r+ M(1-t_i)) 1_K, \qquad i\in[n], \\
& 1_n^\top t\ge c_n, \quad z\in\mathcal Z, \quad t\in\{0,1\}^n.
\end{aligned}
\end{equation}

This is exactly~\eqref{eq:finite-conformal-robust}.

\subsection{Proof of Corollary~\ref{corollary:reform-linear}}
\label{app:linear-loss}

\begin{proof}
By Proposition~\ref{prop:efficient-solution}, the robust part of \eqref{eq:finite-conformal-robust} can be written as
\[
    \min_{W,b,z,r,t}
    \left\{
    \max_{u}
    \left\{
        g_0(z)+z^\top Qu:
        W(u-\hat f(x))+b\le r1_K
    \right\}
    \right\},
\]
together with the calibration constraints
\[
    W\epsilon_i+b\le (r+M(1-t_i))1_K,\qquad i\in[n],
    \qquad
    1_n^\top t\ge c_n,\qquad t\in\{0,1\}^n .
\]
Since $g_0(z)$ does not depend on $u$, the inner maximization is equivalent to
\[
    g_0(z)
    +
    \max_u
    \left\{
        z^\top Qu:
        W(u-\hat f(x))+b\le r1_K
    \right\}.
\]
The constraint can be rewritten as
\[
    Wu \le r1_K-b+W\hat f(x).
\]
Therefore, the inner maximization is the linear program
\[
    \max_u
    \left\{
        (Q^\top z)^\top u:
        Wu\le r1_K-b+W\hat f(x)
    \right\}.
\]
Introducing dual variables $\lambda\in\mathbb R_+^K$ associated with the constraints $Wu\le r1_K-b+W\hat f(x)$, the dual problem is
\[
    \min_{\lambda\in\mathbb R_+^K}
    \left\{
        \bigl(r1_K-b+W\hat f(x)\bigr)^\top \lambda:
        W^\top \lambda = Q^\top z
    \right\}.
\]
Under strong duality, which holds whenever the inner linear program is feasible and has a finite optimum, the inner maximization is equal to its dual. Hence, the min--max problem is equivalent to
\[
\begin{aligned}
\min_{W,b,z,\lambda,r,t}\quad
& g_0(z)
+
\bigl(r1_K-b+W\hat f(x)\bigr)^\top \lambda \\
\text{s.t.}\quad
& W^\top\lambda = Q^\top z,\qquad \lambda\in\mathbb R_+^K,\\
& W\epsilon_i+b \le (r+M(1-t_i))1_K, \qquad i\in[n],\\
& 1_n^\top t\ge c_n,\\
& z\in\mathcal Z,\quad t\in\{0,1\}^n .
\end{aligned}
\]
\end{proof}

\subsection{Polyhedral Conformal Set Learning}
\label{app:pinball-quantile}

In this subsection, write $\mathcal{C}_\theta(x,r)\coloneqq\{y\in\mathcal{Y}:s(x,y;\theta)\le r\}$ for the set at a specified radius.

This section elaborates on the two optimization algorithms described in the main text: the Column-and-Constraint Generation (CCG) method (Alg.~\ref{alg:ccg}) for general concave objectives, and the stochastic pinball-subgradient method (Alg.~\ref{alg:pinball}) for scalable training.

\paragraph{CCG learning.}

Alg.~\ref{alg:ccg} solves the problem~\eqref{eq:finite-conformal-robust} by iteratively growing a finite scenario set $\mathcal{S}_h\subseteq\mathcal{Y}$. At CCG iteration $h$, the restricted master problem is:
\begin{equation}
\label{eq:general-master}
\begin{aligned}
\mathrm{MP}_h:\quad
\min_{W,b,z,r,t,\eta}\quad & \eta \\
\text{s.t.}\quad
& g(z,y^j)\le\eta,\qquad\forall\,y^j\in\mathcal{S}_h,\\
& W(y^j-\hat{f}(x))+b\le r1_K,\qquad\forall\,y^j\in\mathcal{S}_h,\\
& W\epsilon_i+b\le\bigl(r+M(1-t_i)\bigr)1_K,\quad i\in[n],\\
& 1_n^\top t\ge c_n,\quad z\in\mathcal{Z},\quad t\in\{0,1\}^n.
\end{aligned}
\end{equation}
Given a master solution $(\theta^h,z^h,r^h,t^h,\eta^h)$ with $\theta^h=(W^h,b^h)$ and conformal radius $r^h$, the subproblem finds the worst-case outcome in the current conformal set $\mathcal{C}_{\theta^h}(x,r^h)$:
\begin{equation}
\label{eq:general-separation}
\mathrm{SP}_h:\quad
\zeta^h
=\max_{y\in\mathcal{Y}}\;g(z^h,y)
\quad\text{s.t.}\quad
W^h\left(y-\hat{f}(x)\right)+b^h\le r^h 1_K.
\end{equation}
When $g(z,\cdot)$ is convex, the maximizer over the polyhedral set can be chosen at an extreme point. CCG terminates when $\zeta^h-\eta^h\le\varepsilon$. Because the master problem contains $n$ binary variables $\{t_i\}$, its complexity can grow with $n$, motivating the pinball surrogate below.

\begin{algorithm}[t]
\caption{CCG for the general concave objective~\eqref{eq:finite-conformal-robust}}
\label{alg:ccg}
\begin{algorithmic}[1]
\Require $\mathcal D_n$, $\varepsilon>0$, $M$
\State Initialize scenario set $\mathcal S_0\gets\emptyset$, iteration counter $h\gets0$
\Repeat
    \State Solve the restricted master problem $\mathrm{MP}_h$ in~\eqref{eq:general-master}
    \State Let $(\theta^h,z^h,r^h,t^h,\eta^h)$ be an optimal master solution
    \State Solve the separation problem $\mathrm{SP}_h$ in~\eqref{eq:general-separation}
    \State Let $\zeta^h$ be its optimal value and $y^{h+1}$ an optimizer
    \If{$\zeta^h-\eta^h\le \varepsilon$}
        \State \Return $(\theta^h,z^h,r^h,t^h)$
    \Else
        \State Add adversarial scenario:
        $
        \mathcal S_{h+1}\gets \mathcal S_h\cup\{y^{h+1}\}
        $
        \State Set $h\gets h+1$
    \EndIf
\Until{termination}
\end{algorithmic}
\end{algorithm}

\begin{algorithm}[t]
\caption{Stochastic pinball-subgradient method for~\eqref{eq:pinball-surrogate-app}}
\label{alg:pinball}
\begin{algorithmic}[1]
\Require Learning data $\mathcal D_n$, stepsizes $\{\eta_\ell\}$, minibatch size $B$, calibration penalty $\gamma$, $\tau_n=c_n/n$
\State Initialize $(\theta^0,r^0,z^0)\in\Theta\times\mathbb R\times\mathcal Z$, with $\theta^0=(W^0,b^0)$;
\For{$\ell=0,1,2,\ldots$}
    \State Sample minibatch $\mathcal B_\ell\subseteq[n]$, $|\mathcal B_\ell|=B$;
    \State
    $y^\ell\in\arg\max_{y\in\mathcal C_{\theta^\ell}(x,r^\ell)}g(z^\ell,y)
    $;
        $\widehat L_{\mathcal B_\ell}(\theta,r)
        =
        \frac1B\sum_{i\in\mathcal B_\ell}
        \rho_{\tau_n}\bigl(S_i(\theta)-r\bigr);
    $
    \State Compute a stochastic subgradient $\widehat G^\ell$ of
    $
        \max_{y\in\mathcal C_\theta(x,r)}g(z,y)+\gamma\widehat L_{\mathcal B_\ell}(\theta,r)
    $
    at $(\theta^\ell,r^\ell,z^\ell)$, using $y^\ell$ and the multipliers of its constraints for the first term and~\eqref{eq:subgradients} for the second;
    \State
$        (\theta^{\ell+1},r^{\ell+1},z^{\ell+1})
        \gets
        \operatorname{Proj}_{\Theta\times\mathbb R\times\mathcal Z}
        \bigl((\theta^\ell,r^\ell,z^\ell)-\eta_\ell\widehat G^\ell\bigr);
    $
\EndFor
\State \Return $\hat\theta=(\hat W,\hat b)$.
\end{algorithmic}
\end{algorithm}

\paragraph{Stochastic learning with pinball loss.}
The exact formulation in~\eqref{eq:finite-conformal-robust} enforces empirical coverage through binary variables that select which calibration residuals are contained in the candidate set. To obtain a scalable continuous approximation, we replace this discrete quantile-selection mechanism with a pinball calibration penalty.

For a fixed set parameter $\theta$, write
\[
    S_i(\theta) \coloneqq s(X_i,Y_i;\theta),
    \qquad i\in[n],
\]
and let $\tau_n\coloneqq c_n/n$. The empirical conformal radius is the smallest value of $r$ such that at least $c_n$ nonconformity scores are no larger than $r$. Equivalently, it is the smallest empirical $\tau_n$-quantile of the scores $\{S_i(\theta)\}_{i=1}^n$. This quantile admits the standard pinball-loss characterization. Define
\[
    \rho_\tau(u)
    \coloneqq
    \tau u_+ + (1-\tau)(-u)_+,
    \qquad
    u_+\coloneqq \max\{u,0\},
\]
and
\[
    L_n(\theta,r)
    \coloneqq
    \frac1n
    \sum_{i=1}^n
    \rho_{\tau_n}\bigl(S_i(\theta)-r\bigr).
\]
For fixed $\theta$, the subdifferential with respect to $r$ is
\[
    \partial_r L_n(\theta,r)
    =
    \frac1n
    \left[
        \#\{i:S_i(\theta)<r\}-\tau_n n,\;
        \#\{i:S_i(\theta)\le r\}-\tau_n n
    \right].
\]
Hence,
\[
    0\in \partial_r L_n(\theta,r)
    \quad\Longleftrightarrow\quad
    \#\{i:S_i(\theta)<r\}
    \le
    \tau_n n
    \le
    \#\{i:S_i(\theta)\le r\}.
\]
Thus, any minimizer of $L_n(\theta,r)$ is an empirical $\tau_n$-quantile of the nonconformity scores, and the smallest minimizer recovers the conformal radius used in Lemma~\ref{lem:radius-bigM}. The pinball loss therefore provides a continuous relaxation of the binary quantile constraint in~\eqref{eq:finite-conformal-robust}.

We solve the following surrogate problem:
\begin{equation}
\label{eq:pinball-surrogate-app}
\min_{\theta\in\Theta,\;r\in\mathbb R,\;z\in\mathcal Z}
\left\{
    \max_{y\in\mathcal C_\theta(x,r)}
    g(z,y)
    +
    \gamma L_n(\theta,r)
\right\},
\end{equation}
where $\gamma>0$ controls the strength of the calibration penalty. The first term learns a polyhedral geometry that is favorable for the downstream robust decision, while the second term keeps the radius aligned with the target empirical conformal quantile. Compared with the exact mixed-integer formulation, this surrogate removes the binary variables and can be optimized with stochastic first-order methods.

For use in Alg.~\ref{alg:pinball}, the subgradient of the pinball term with respect to its argument is
\begin{equation}
\label{eq:subgradients}
    \xi_i
    \in
    \partial \rho_{\tau_n}\bigl(S_i(\theta)-r\bigr)
    =
    \begin{cases}
        \tau_n, & S_i(\theta)>r,\\
        [-(1-\tau_n),\,\tau_n], & S_i(\theta)=r,\\
        -(1-\tau_n), & S_i(\theta)<r .
    \end{cases}
\end{equation}
Consequently, scores above the current radius push $r$ upward, while scores below the current radius push it downward. When differentiating with respect to $\theta$, the same coefficients $\xi_i$ weight the subgradients of the scores $S_i(\theta)$, allowing the set geometry and radius to be updated jointly.

\section{Proofs in Section~\ref{sec:theory}}
\label{app:proofs-theory}

\subsection{Proof for Lemma~\ref{thm:validity}}

\begin{proof}
Since the nonconformity score $s(\cdot,\cdot;\theta)$ is deterministic for fixed $\theta$, and we assume that the augmented sample $ \{(X_i,Y_i)\}_{i=1}^{n+1} $ is exchangeable, the induced score sequence $ \left\{ s(X_i,Y_i;\theta) \right\}_{i=1}^{n+1} $ is also exchangeable. Therefore, the rank of $s(X_i,Y_i;\theta)$ among the $n+1$ scores (with ties broken uniformly at random) is uniformly distributed over $\{1, \ldots, n+1\}$. By the definition of the conformal radius $R(\mathcal{D}_n;\theta)$ in \eqref{eq:radius}, there is
    $$
    \mathbb{P}
    \left\{
    s(X_{n+1},Y_{n+1};\theta)
    \leq
    R(\mathcal{D}_n;\theta)
    \right\}
    \geq
    1-\alpha.
    $$
Using the definition of $\mathcal{C}(x;\mathcal{D}_n,\theta)$ completes the proof:
    $$
    \mathbb{P}
    \left\{
    Y_{n+1}
    \in
    \mathcal{C}(X_{n+1};\mathcal{D}_n,\theta)
    \right\}
    \geq
    1-\alpha.
    $$
\end{proof}

\subsection{Proof for Lemma~\ref{thm:invalid}}

\begin{proof}
Since $\hat\theta(\mathcal{D}_n)$ is learned from $\mathcal{D}_n$, conditioning on $\hat\theta(\mathcal{D}_n)$ reveals information about the calibration samples $(X_1,Y_1),\ldots,(X_n,Y_n)$, but not about the test point $(X_{n+1},Y_{n+1})$. Consequently, the conditional score sequence
    $$
    s(X_1,Y_1;\hat\theta),\ldots,s(X_n,Y_n;\hat\theta),
    s(X_{n+1},Y_{n+1};\hat\theta)
    \mid \hat\theta=\hat\theta(\mathcal{D}_n)
    $$
is generally no longer exchangeable. The loss of exchangeability arises from the asymmetric role of the test point relative to the calibration samples, which is precisely the type of nonexchangeability studied in~\cite{barber2023conformal}.

Following the argument of Theorem~2 in~\cite{barber2023conformal}, call an entry of a sequence of $n+1$ data points \emph{strange} if its nonconformity score, computed with $\hat\theta$ fitted on the first $n$ entries of that sequence, exceeds the $c_n$-th smallest score of the other $n$ entries. At most $n+1-c_n\le\alpha(n+1)$ entries of any sequence are strange, and the test point is not covered exactly when the $(n+1)$-th entry of $\mathcal{D}_{n+1}$ is strange. Since $\mathcal{D}_{n+1}^{(i)}$ has the same distribution as $\mathcal{D}_{n+1}$, the probability that the $(n+1)$-th entry of $\mathcal{D}_{n+1}^{(i)}$, namely $(X_i,Y_i)$, is strange equals the miscoverage probability. Computing the scores with $\hat\theta$ fitted on the first $n$ entries of $\mathcal{D}_{n+1}$ instead changes this probability by at most $d_{\rm TV}(\mathcal{D}_{n+1},\mathcal{D}_{n+1}^{(i)})$, and turns it into the probability that the $i$-th entry of $\mathcal{D}_{n+1}$ is strange. Averaging over the $n+1$ entries,
    $$
    \begin{aligned}
    \mathbb{P}\bigl\{ Y_{n+1} \notin \mathcal{C}(X_{n+1};\mathcal{D}_n,\hat\theta(\mathcal{D}_n)) \bigr\}
    &\le \frac{1}{n+1}\,\mathbb{E}\bigl[\text{number of strange entries of }\mathcal{D}_{n+1}\bigr]\\
    &\quad+\frac{1}{n+1}\sum_{i=1}^{n} d_{\rm TV}\left(\mathcal{D}_{n+1},\mathcal{D}_{n+1}^{(i)}\right)\\
    &\le \alpha+\frac{1}{n+1}\sum_{i=1}^{n} d_{\rm TV}\left(\mathcal{D}_{n+1},\mathcal{D}_{n+1}^{(i)}\right).
    \end{aligned}
    $$
Taking complements gives the claim. When $\hat\theta$ does not depend on $\mathcal{D}_n$, both fits coincide, the distance is zero, and the usual conformal guarantee is recovered.
This completes the proof.
\end{proof}

\subsection{Proof for Theorem~\ref{thm:re-calibrated-validity}}

\begin{proof}
Conditional on the learning data $\mathcal{D}_n$, the learned parameter $ \hat\theta(\mathcal{D}_n) $ is fixed. By assumption, the recalibration dataset $ \mathcal{D}_m=\{(X_j,Y_j)\}_{j=1}^m $ and the test point $(X_{m+1},Y_{m+1})$ are exchangeable with $\mathcal{D}_n$. Therefore, conditional on $\mathcal{D}_n$, the augmented score sequence
    $$
    s(X_1,Y_1;\theta),\ldots,
    s(X_m,Y_m;\theta),
    s(X_{m+1},Y_{m+1};\theta) \mid \theta = \hat\theta(\mathcal{D}_n)
    $$
is exchangeable. Using Lemma~\ref{thm:validity}, there is:
    $$
    \mathbb{P}
    \left\{ Y_{m+1} \in \mathcal{C}(X_{m+1}; \mathcal{D}_m, \hat \theta(\mathcal{D}_n)) \mid \hat \theta(\mathcal{D}_n)\right\} \geq 1- \alpha.
    $$
Taking expectation with respect to $\mathcal{D}_n$ yields
    $$
    \mathbb{P}
    \left\{
    Y_{m+1}
    \in
    \mathcal{C}(X_{m+1};\mathcal{D}_m,\hat\theta(\mathcal{D}_n))
    \right\}
    \geq
    1-\alpha.
    $$
\end{proof}

\subsection{Proof for Theorem~\ref{thm:gap}}
\label{app:gap}

\begin{proof}[Proof of Theorem~\ref{thm:gap}]
We begin by defining the calibration error and the parameter error in \eqref{eq:cal-error} and \eqref{eq:para-error}, respectively:
    \begin{align}
        \Delta(\mathcal{D}_m) & =  \mathcal{L}\left( \hat\theta(\mathcal{D}_n), R(\mathcal{D}_m;\hat \theta(\mathcal{D}_n)) \right) -
        \mathcal{L}\left( \hat\theta(\mathcal{D}_n), R^\star(\hat \theta(\mathcal{D}_n)) \right),
        \label{eq:cal-error}
        \\
        \Delta(\mathcal{D}_n) & =  \mathcal{L}\left( \hat\theta(\mathcal{D}_n), R^\star(\hat \theta(\mathcal{D}_n)) \right)
        - \mathcal{L}\Big( \theta^\star, R^\star(\theta^\star) \Big).
        \label{eq:para-error}
    \end{align}
Then, the suboptimality gap is upper bounded by the sum of two errors:
    $$
    \Delta = \Delta(\mathcal{D}_m) + \Delta(\mathcal{D}_n) \leq |\Delta(\mathcal{D}_m)| + \Delta(\mathcal{D}_n).
    $$
We proceed with analyzing the bound by expanding the two terms respectively.

    \paragraph{Calibration error $\Delta(\mathcal{D}_m)$}
For any $\theta \in \Theta$, note that the data-driven conformal radius $R(\mathcal{D}_m; \theta)$ is the empirical counterpart of the population quantile $R^\star(\theta)$. Denote the cumulative distribution function (CDF) that they each represent by $\hat F_m: \mathbb{R} \to [0, 1]$ and $F^\star: \mathbb{R} \to [0, 1]$, respectively. Using the Dvoretzky–Kiefer–Wolfowitz (DKW) inequality~\citep{dvoretzky1956asymptotic}, there is:
    $$
    \mathbb{P}\left\{ \| \hat F_m - F^\star \|_\infty \ge \epsilon \right\} \leq 2e^{-2 m \epsilon^2} \Rightarrow \mathbb{P}\left\{ \| \hat F_m - F^\star \|_\infty \le \sqrt{\frac{\log\left(2 / \delta \right)}{2m}} \right\} \ge 1 - \delta.
    $$
Using $(b)$ in Assumption~\ref{ass:gap}, we know that $F^\star$ has a derivative at least $\kappa$ near $R^\star(\theta)$. Since $R(\mathcal{D}_m;\theta)$ is the empirical quantile at level $c_m/m\in[1-\alpha,1-\alpha+2/m)$, with $c_m=\lceil(m+1)(1-\alpha)\rceil$, inverting the CDFs gives, with probability greater than $1-\delta$ and for $m$ large enough that the radii lie in the neighborhood in $(b)$:
    $$
    | R(\mathcal{D}_m; \theta) - R^\star(\theta) | \leq \frac{1}{\kappa}\Bigl( \| \hat F_m - F^\star \|_\infty + \frac{2}{m}\Bigr) \leq \frac{1}{\kappa}\Bigl( \sqrt{\frac{\log\left(2 / \delta \right)}{2m}} + \frac{2}{m}\Bigr).
    $$
By using the Lipschitz condition $(c)$ in Assumption~\ref{ass:gap} and applying this bound conditionally on $\mathcal{D}_n$, which is independent of $\mathcal{D}_m$, there is:
    $$
    |\Delta(\mathcal{D}_m)| \leq L_r \cdot | R(\mathcal{D}_m; \hat\theta(\mathcal{D}_n)) - R^\star(\hat\theta(\mathcal{D}_n)) | \leq \frac{L_r}{\kappa}\Bigl( \sqrt{\frac{\log\left(2 / \delta \right)}{2m}} + \frac{2}{m}\Bigr).
    $$

    \paragraph{Parameter error $\Delta(\mathcal{D}_n)$}

The parameter error $\Delta(\mathcal{D}_n)$ can be further decomposed into three error terms:
    $$
    \Delta(\mathcal{D}_n) = \Delta_1 + \Delta_2 + \Delta_3 \leq |\Delta_1| + \Delta_2 + |\Delta_3|,
    $$
where
    \begin{align*}
        \Delta_1
        & = \mathcal{L}\left( \hat\theta(\mathcal{D}_n), R^\star(\hat \theta(\mathcal{D}_n)) \right)
        - \mathcal{L}\left( \hat\theta(\mathcal{D}_n), R(\mathcal{D}_n ;
        \hat \theta(\mathcal{D}_n)) \right), \\
        \Delta_2
        & = \mathcal{L}\left( \hat\theta(\mathcal{D}_n), R(\mathcal{D}_n ;
        \hat \theta(\mathcal{D}_n)) \right)
        - \mathcal{L}\Big( \theta^\star, R(\mathcal{D}_n ; \theta^\star) \Big), \\
        \Delta_3
        & = \mathcal{L}\Big( \theta^\star, R(\mathcal{D}_n ; \theta^\star) \Big)
        - \mathcal{L}\Big( \theta^\star, R^\star(\theta^\star) \Big).
    \end{align*}
The second error term $\Delta_2$ can be bounded using condition ($d$) of Assumption~\ref{ass:gap}:
    \begin{align*}
    \Delta_2 \leq
    \mathcal{L}\left( \hat\theta(\mathcal{D}_n), R(\mathcal{D}_n;\hat \theta(\mathcal{D}_n)) \right) - \inf_{\theta \in \Theta} \mathcal{L}\left( \theta, R(\mathcal{D}_n;\theta) \right)  \leq \epsilon_{\rm opt}.
    \end{align*}
Using condition $(c)$ of Assumption~\ref{ass:gap}, the first and third error terms share the same upper bound:
    \begin{align*}
        |\Delta_1|
        & \leq L_r \cdot \left| R^\star(\hat \theta(\mathcal{D}_n)) - R(\mathcal{D}_n; \hat \theta(\mathcal{D}_n)) \right| \leq L_r \cdot  \sup_{\theta \in \Theta} \left| R^\star(\theta) - R(\mathcal{D}_n; \theta) \right| \\
        |\Delta_3| & \leq L_r \cdot \left| R^\star(\theta^\star) - R(\mathcal{D}_n; \theta^\star) \right| \leq L_r \cdot \sup_{\theta \in \Theta} \left| R^\star( \theta) - R(\mathcal{D}_n; \theta) \right|.
    \end{align*}
For any $\theta$, the conformal radius $R(\mathcal{D}_n; \theta)$ is an empirical approximation to the population quantile $R^\star(\theta)$ using $n$ samples. We redefine the cumulative distribution function (CDF) that they each represent by $\hat F_n(\cdot;\theta): \mathbb{R} \to [0, 1]$ and $F^\star(\cdot;\theta): \mathbb{R} \to [0, 1]$, respectively. Combined with the condition $(a)$ of Assumption~\ref{ass:gap}, using the uniform convergence of VC classes~\citep{vapnik2015uniform}, there is:
    $$
    \mathbb{P}
    \left\{
    \sup_{\theta \in \Theta,\ r \in \mathbb{R}}
    \left| \hat F_n(r;\theta)-F^\star(r;\theta) \right|
    \le \epsilon
    \right\}
    \ge
    1-
    8
    \left(\frac{en}{V}\right)^V
    \exp\left(-\frac{n\epsilon^2}{32}\right).
    $$
By reorganizing the terms, we conclude that with probability greater than $1-\delta$:
    $$
    \sup_{\theta \in \Theta}
    \left\| \hat F_n(\cdot;\theta)-F^\star(\cdot;\theta) \right\|_\infty
    \le
    \sqrt{
    \frac{
    32\left[
    V\log(en/V)+\log(8/\delta)
    \right]
    }{n}
    }.
    $$
Again, using $(b)$ in Assumption~\ref{ass:gap} and the same argument with level $c_n/n\in[1-\alpha,1-\alpha+2/n)$, we obtain that with probability $1-\delta$:
    $$
    \sup_{\theta \in \Theta}
    \left|
     R(\mathcal{D}_n;\theta)-R^\star(\theta) \right|
    \le
    \frac{1}{\kappa}\Bigl(
    \sqrt{
    \frac{
    32\left[
    V\log(en/V)+\log(8/\delta)
    \right]
    }{n}
    }+\frac{2}{n}\Bigr).
    $$
Therefore, combining the three error terms, there is with probability greater than $1-\delta$:
    $$
    \Delta(\mathcal{D}_n) \leq \frac{2 L_r}{\kappa}\Bigl(
    \sqrt{
    \frac{
    32\left[
    V\log(en/V)+\log(8/\delta)
    \right]
    }{n}
    }+\frac{2}{n}\Bigr) + \epsilon_{\rm opt}.
    $$

    \paragraph{Combining $\Delta(\mathcal{D}_m)$ and $\Delta(\mathcal{D}_n)$}
Combining the calibration and parameter errors and using a union bound adjustment by setting $\delta$ to $\delta/2$ in the previous terms, there is, with probability greater than $1-\delta$:
    $$
    \Delta \leq
    \frac{L_r}{\kappa} \sqrt{\frac{\log\left(4 / \delta \right)}{2m}} +
    \frac{2 L_r}{\kappa}
    \sqrt{
    \frac{
    32\left[
    V\log(en/V)+\log(16/\delta)
    \right]
    }{n}
    } + \frac{2L_r}{\kappa}\Bigl(\frac{1}{m}+\frac{2}{n}\Bigr) + \epsilon_{\rm opt}.
    $$
Using the big-$\mathcal{O}_\mathbb{P}$ notation~\citep{van2000asymptotic}, we conclude that:
    $$
    \Delta \leq \mathcal{O}_\mathbb{P}\left( \sqrt{\frac{1}{m}} + \sqrt{\frac{\log n}{n}}\right) + \epsilon_{\rm opt}.
    $$

\end{proof}

\section{Additional Experimental Details and Results}
\label{app:experiments}

\subsection{Synthetic Experimental Details}
\label{app:exp-details}

\paragraph{Data splits and evaluation.}
For the synthetic experiments, we use a fixed four-way split: 1,200 samples for fitting the predictor $\hat f$, a learning set $\mathcal D_n$ of 300 samples for learning the set geometry $\theta$, a calibration set $\mathcal D_m$ of 300 samples for determining the final conformal radius, and a held-out test set of 1,000 samples used only for evaluation. Results are averaged over 20 random seeds (five for the sensitivity analysis), each corresponding to an independent draw of the full dataset. Baselines that do not require a separate learning split (Bonferroni and conformal box/ball) use the combined pool $\mathcal D_n\cup\mathcal D_m$ of 600 samples for calibration, ensuring a fair comparison. We set the nominal miscoverage level to $\alpha=0.10$ (corresponding to a target coverage of $1-\alpha=0.90$) and use $K=4$ half-planes for all polyhedral sets.

\paragraph{Synthetic data generation and predictor.}
For the synthetic tasks, covariates $X\in\mathbb R^4$ are drawn from $\mathcal N(0,I_4)$, and outcomes $Y\in\mathbb R^2$ follow a heteroscedastic nonlinear model:
\[
    Y = \mu(X) + \varepsilon(X),
\]
where $\mu_1(X)=0.75X_1+0.35\sin(X_2)-0.20X_3$, $\mu_2(X)=-0.45X_1+0.55X_2^2+0.30X_4$, and $\varepsilon(X)$ is drawn from a zero-mean Gaussian with a covariate-dependent scale factor $0.80+0.45\,\sigma(X_1)$ (where $\sigma$ is the sigmoid function) applied to a fixed anisotropic root-covariance matrix. An additional 6\% heavy-tail contamination is added to each draw. This design produces residuals that are both anisotropic and heteroscedastic, so that the set shape has a meaningful effect on coverage and downstream cost.

We fit the predictor $\hat f$ by ridge regression on a nonlinear feature map $\phi(x)=[1,\,x^\top,\,x_1^2,\,x_2^2,\,\sin(x_1),\,\sin(x_2)]^\top\in\mathbb R^9$, with ridge penalty $\lambda=10^{-3}$. Residuals $\epsilon_i=Y_i-\hat f(X_i)$ on the learning and calibration splits serve as the input to all conformal methods.

\paragraph{Synthetic objectives.}
We evaluate three two-dimensional robust decision tasks. The decision variable $z\in\Delta_2$ lives on the probability simplex, i.e.\ $z_1+z_2=1$, $z_j\ge0$.

\noindent\emph{Linear objective.}
\[
    g(z,u)=(c+u)^\top z,\qquad c=(-0.08,\,0).
\]
The constant $c_1=-0.08$ gives the first action a slight structural advantage; the sign of the realized outcome determines whether this advantage holds, making set orientation consequential.

\noindent\emph{Quadratic-regularized objective.}
\[
    g(z,u)=(c+u)^\top z+\beta(z_1-z_2)^2,\qquad c=(-0.08,\,0),\quad\beta=0.08.
\]
The quadratic term penalizes unequal allocation, introducing a smoothly curved Pareto frontier between the two actions.

\noindent\emph{Newsvendor objective.}
\[
    g(z,u)
    =
    \sum_{j=1}^2
    \bigl[
        c_o\max(z_j-u_j,0)
        +
        c_u\max(u_j-z_j,0)
    \bigr],
    \qquad c_o=0.3,\quad c_u=0.7.
\]
Here $z\in\Delta_2$ is the order quantity and $u$ is the realized demand. Overage ($c_o=0.3$) and underage ($c_u=0.7$) costs are asymmetric. The loss is piecewise linear and nonsmooth; its worst-case value depends on which tail of the residual distribution the uncertainty set exposes, making the learned set shape particularly consequential.

\subsection{Real Data Experimental Details}

\paragraph{Data splits.}
We use the UCI Energy Efficiency dataset~\citep{energy_efficiency_242}, also used by~\citet{pmlr-v206-wang23n} for conformal prediction evaluation. We partition the data without replacement into 320 training, 160 learning, 144 recalibration, and 144 test observations for each seed. The four subsets are mutually disjoint within each run, while the 80 seeds correspond to independently repeated random partitions of the same dataset; we use more partitions than for the synthetic tasks because each one has only 144 calibration and 144 test observations.

\paragraph{Evaluation and objectives.}
Unlike prior evaluations of this dataset that report only coverage and set volume, we measure downstream worst-case cost and regret. The covariates $X\in\mathbb R^8$ encode eight building-design attributes (relative compactness, surface area, wall area, roof area, height, orientation, glazing area, glazing-area distribution), and $Y=(Y_1,Y_2)$ are heating and cooling loads (kWh/m$^2$). We pair the dataset with the linear objective $g(z,u)=(c+u)^\top z$, $c=(-0.08,0)$, where $z\in\Delta_2$ allocates a unit budget between heating ($z_1$) and cooling ($z_2$) control. We fit $\hat f$ by the same ridge feature regression as in the synthetic experiments.

\subsection{Baseline Implementation Details}

We compare against decision-agnostic conformal baselines, including Bonferroni boxes~\citep{dunn1961multiple}, conformal boxes or balls~\citep{vovk2005algorithmic}, and PCP~\citep{pmlr-v206-wang23n}. We also include decision-aware baselines based on F-CROMS~\citep{bao2025optimal} and PICNN scores~\citep{yeh2024end}. For ablations, ``Polyhedron'' calibrates a polyhedral set without using the downstream loss, ``Polyhedron (min size)'' learns the polyhedral set by minimizing volume, and ``Ours (w/o recal)'' uses the same learned geometry $\hat\theta(\mathcal D_n)$ as our method but sets the radius on $\mathcal D_n$ itself, without an independent post-selection calibration step. All methods are evaluated over multiple random seeds using the same train--learn--calibrate--test splits whenever applicable.

\paragraph{Ours.}
\label{app:ours-impl}
For each learning problem, we first attempt the exact mixed-integer quadratically constrained program (MIQCP) of Corollary~\ref{corollary:reform-linear} with Gurobi~\citep{gurobi} under a 30-minute wall-clock limit. The MIQCP averages the robust loss over the first 120 learning contexts, i.e., it approximates the empirical objective defining $\hat\theta(\mathcal D_n)$, and is posed over the compact class of paired normals $W=(a_1,a_2,-a_1,-a_2)^\top$ with $A=[a_1;a_2]$ symmetric, $2\|A\|_F^2=1$, $\lambda_{\min}(A)\ge 0.15$, and $b=0$; we substitute $\mu=r\lambda$ for the dual variables, and the quadratic-regularized objective only adds the convex term $\beta(z_1-z_2)^2$ to $g_0(z)$. We use the MIQCP solution only if Gurobi certifies optimality (relative gap at most $0.5\%$ or absolute gap at most $10^{-3}$). Otherwise, and whenever the reformulation does not apply (newsvendor loss, $K\neq 4$, or $d>2$), we use the pinball-loss gradient method: for $d=2$, gradient descent on the smoothed pinball surrogate over the rotation, aspect ratio, and radius of a regular $K$-gon (five starts, 80 iterations); for $d>2$, the stochastic pinball-subgradient method of Alg.~\ref{alg:pinball} over $W=DM$, where $D$ is a fixed positively spanning set of $K$ directions and $M\in\mathbb R^{d\times d}$ is learned. No variant selects from a fixed library of candidate sets. In the main experiments, Gurobi certified the MIQCP within 1--29 minutes in 50 of the 80 energy-efficiency runs and within 3--9 minutes in 3 of the 20 linear-objective runs; all quadratic-regularized runs and the remaining linear and energy-efficiency runs reached the limit with absolute gaps between $0.001$ and $0.15$, and the newsvendor task has no MIQCP reformulation. In the sensitivity study, the MIQCP was certified within 1--2 minutes in two of the five runs with $n=120$ and reached the limit at all other $K=4$, $d=2$ points. In all cases the final radius is recalibrated on $\mathcal D_m$, and reported times include the full geometry-learning time, including the 30-minute MIQCP attempt whenever it is made.

\paragraph{F-CROMS.}
F-CROMS~\citep{bao2025optimal} selects a conformal prediction set from a finite library of candidate sets, each induced by a different nonconformity score.  In the original paper the library consists of model-induced scores; here we instantiate the library with randomly generated halfspace cuts to give the method a fair chance to discover polyhedral geometries similar to ours: each candidate intersects $K$ halfspaces with random unit normals, $s_k(x,y)=\max_{j\le K} w_{kj}^\top(y-\hat f(x))$. For each $k$ we calibrate the threshold on the learning split $\mathcal D_n$ via split conformal prediction, obtaining a polyhedral prediction set. F-CROMS then picks the candidate $k^\star$ whose set minimizes the empirical worst-case robust objective on $\mathcal D_n$ and recalibrates it on $\mathcal D_m$. We evaluate two library sizes: F-CROMS (10) with 10 candidates and F-CROMS (20) with 20. This implementation is faithful to the original method's spirit: it can only select among a discrete pool of pre-specified directions and cannot optimize the halfspace orientation continuously.  Our method, by contrast, optimizes the set geometry over a continuous parameter space rather than selecting it from a finite pool.

\paragraph{PCP.}
PCP~\citep{pmlr-v206-wang23n} draws $N$ samples $\hat y_1,\dots,\hat y_N$ from a conditional generative model of $Y$ given $X=x$, uses the distance to the nearest sample, $s(x,y)=\min_{k\le N}\|y-\hat y_k\|_2$, as the nonconformity score, and returns the union of the $N$ balls of radius $\hat q$ centered at the samples. Following the authors' setting for two-dimensional outcomes, the generative model is a mixture density network (10 Gaussian components with diagonal covariances, two hidden layers of width 16) for $Y-\hat f(X)$ given $\hat f(X)$, fitted by maximum likelihood on $\mathcal D_n$ with early stopping on a held-out fifth. We use the high-density version proposed in the same paper (HD-PCP), which draws $\lceil N/0.8\rceil$ samples per covariate and keeps the $N=1000$ with the highest estimated density; without this filtering, the far-tail samples dominate the worst case. The radius $\hat q$ is calibrated on $\mathcal D_m$. Because the losses are convex in $y$, the worst case over the union of balls equals the worst case over its convex hull, whose support function is $\max_{k} a^\top\hat y_k+\hat q\|a\|_2$. We report the area of the union, averaged over the calibration covariates.

\paragraph{Yeh et al.}
\citet{yeh2024end} train the uncertainty set end-to-end by parameterizing the nonconformity score via a partially input-convex neural network (PICNN), which guarantees that the sublevel set $\{y : s(x,y;\theta)\le\tau\}$ is convex in $y$ for every $x$. Gradients with respect to $\theta$ flow through the adversarial argmax via KKT conditions and implicit differentiation, enabling joint learning of set shape and size. We implement this baseline following the authors' code, using their PICNN score class with a two-hidden-layer architecture (width 64, ReLU activations) that takes $\hat f(x)$ as input and the standardized residual as its convex argument, plus the term $\|y\|_\infty$ that the authors propose to keep the sublevel sets bounded (without it, the learned sets were unbounded on the energy data). The score is pretrained with the authors' estimate-then-optimize energy-based fit (Adam, step size $10^{-2}$, 50 epochs) and then trained end-to-end on the realized task loss: in each minibatch, half of the points set the conformal threshold, and the robust decisions for the other half are differentiated through the dual linear program with \texttt{cvxpylayers}~\citep{agrawal2019differentiable} (for the newsvendor loss, each affine piece is dualized separately). We use Adam with the step size selected from $\{10^{-3},10^{-4}\}$ and early stopping (at most 100 epochs) on a held-out fifth of $\mathcal D_n$. The conformal radius is then calibrated on the same held-out split as ours so that the final coverage guarantee is comparable; since these sets vary with $x$, we report their volume averaged over the calibration covariates. This design is sensible when $g(z,\cdot)$ is differentiable and the gradient signal can propagate through the PICNN; it produces general convex sets and is thus potentially more expressive than our polyhedral parameterization.

\subsection{Comparison of the Pinball-Loss Approximation and Exact MIQCP}
\label{app:pinball-comparison}

We additionally compare the exact MIQCP formulation with its pinball-loss approximation as the learning-set size $n=|\mathcal D_n|$ increases. This experiment isolates scalability with respect to sample size at a fixed outcome dimension. For each value of $n$, we run both formulations using the same five random seeds and report their average runtimes and worst-case objectives. The performance gap is the percentage increase in the worst-case objective of the pinball solution relative to the solution returned by the MIQCP formulation. Both solutions undergo the same independent conformal recalibration before evaluation.

This study uses the fixed-context specialization of the MIQCP at the mean learning context and the stochastic pinball-subgradient method of Alg.~\ref{alg:pinball}, both over the compact paired-normal class of Appendix~\ref{app:ours-impl}. Each MIQCP is run with a 30-minute limit, within which Gurobi certified global optimality (relative gap at most $0.5\%$) in all 30 runs; Table~\ref{tab:pinball-comparison} reports the number of certified runs for each $n$.

\begin{table}[t]
    \centering
\caption{Comparison of the exact MIQCP formulation and the pinball-loss approximation. Runtimes are averaged over five seeds. The time ratio is MIQCP runtime divided by pinball runtime, so a value greater than one indicates that the pinball approximation is faster.}
    \label{tab:pinball-comparison}
    \small
    \begin{tabular}{rccccc}
        \toprule
        $n$ & Certified runs & MIQCP time (s) & Pinball time (s)
        & Time ratio & Performance gap \\
        \midrule
         20 & $5/5$ &   0.10 & 1.68 & $0.06\times$ & 1.0\% \\
         40 & $5/5$ &   0.13 & 1.72 & $0.08\times$ & 9.3\% \\
         80 & $5/5$ &   0.87 & 2.04 & $0.43\times$ & 7.3\% \\
        160 & $5/5$ &   4.00 & 2.05 & $1.95\times$ & 7.6\% \\
        300 & $5/5$ & 134.12 & 2.00 & $67.22\times$ & 4.3\% \\
        600 & $5/5$ & 325.69 & 2.00 & $163.20\times$ & 4.7\% \\
        \bottomrule
    \end{tabular}
\end{table}

The exact MIQCP is faster for very small learning sets but becomes substantially more expensive as $n$ increases. Its average runtime grows from $0.10$ seconds at $n=20$ to $325.69$ seconds at $n=600$ (up to $1234$ seconds for a single run), whereas the pinball runtime remains stable between $1.68$ and $2.05$ seconds. At $n=300$ and $n=600$, the pinball approximation is approximately $67$ and $163$ times faster, respectively. Meanwhile, the increase in the worst-case objective remains below $10\%$ in all settings and falls below $5\%$ for the two largest instances. These results indicate that the exact formulation remains practical for small instances, while the pinball approximation offers substantially more stable computation with a modest loss in objective performance.

\subsection{Additional Results}

Tables~\ref{tab:combined_main_results} and~\ref{tab:energy_efficiency_linear_auto} provide additional results. Across the three synthetic objectives, the proposed recalibrated method maintains empirical coverage close to the nominal $90\%$ level while achieving the best or near-best worst-case cost. In the linear and quadratic-regularized objectives, our method substantially improves worst-case cost relative to standard conformal boxes, conformal balls, and fixed polyhedral baselines; it also improves on PCP in the linear case and matches it in the quadratic-regularized case. For the newsvendor task, PCP attains the lowest worst-case cost, but our method remains competitive while using a decision-aware learned geometry. These results further show that smaller volume alone does not necessarily imply better downstream performance: PCP and the minimum-size polyhedron often produce compact sets but can incur larger worst-case costs.

\begin{table}[!t]
\centering
\scriptsize
\caption{Comparison across synthetic objectives. Mean\,$\pm$\,std over 20 seeds; $n=m=300$, $K=4$, $1{-}\alpha=0.90$. WC cost and regret are in units of $\times 10^{-1}$; time is total wall-clock seconds per run. For Ours and Ours (w/o recal), time includes the 30-minute MIQCP attempt whenever it is made (Appendix~\ref{app:ours-impl}).}
\label{tab:combined_main_results}
\setlength{\tabcolsep}{4pt}
\resizebox{\textwidth}{!}{%
\begin{tabular}{llccccc}
\toprule
\textbf{Task}
& \textbf{Method}
& \textbf{Coverage}
& \textbf{Volume}
& \textbf{WC cost}
& \textbf{Regret}
& \textbf{Time (s)} \\
\midrule

\multirow{11}{*}{\parbox{2.3cm}{Linear objective}}
& Bonferroni box & $0.914\pm0.019$ & $0.86\pm0.17$ & $1.00\pm0.53$ & $0.27\pm0.13$ & $<0.01$ \\
& Conformal ($\ell_\infty$) & $0.902\pm0.016$ & $1.13\pm0.17$ & $1.48\pm0.57$ & $0.15\pm0.09$ & $<0.01$ \\
& Conformal ($\ell_2$) & $0.902\pm0.019$ & $1.05\pm0.07$ & $1.78\pm0.39$ & $0.34\pm0.09$ & $<0.01$ \\
& PCP & $0.896\pm0.023$ & $0.60\pm0.11$ & $0.74\pm0.75$ & $0.44\pm0.21$ & $1.3$ \\
\cmidrule(lr){2-7}
& F-CROMS (10) & $0.898\pm0.019$ & $3.38\pm4.64$ & $1.64\pm0.45$ & $0.69\pm0.41$ & $0.28$ \\
& F-CROMS (20) & $0.897\pm0.022$ & $2.30\pm3.49$ & $1.30\pm0.55$ & $0.62\pm0.50$ & $0.55$ \\
& Yeh et al. & $0.898\pm0.014$ & $0.88\pm0.18$ & $1.46\pm0.74$ & $0.34\pm0.19$ & $881$ \\
& Polyhedron & $0.902\pm0.020$ & $0.68\pm0.15$ & $1.40\pm0.84$ & $0.49\pm0.31$ & $<0.01$ \\
& Polyhedron (min size) & $0.902\pm0.020$ & $0.59\pm0.08$ & $1.29\pm1.06$ & $0.49\pm0.32$ & $0.05$ \\
\cmidrule(lr){2-7}
& Ours (w/o recal) & $0.905\pm0.020$ & $0.76\pm0.15$ & $\underline{0.64}\pm0.61$ & $0.34\pm0.24$ & $1599$ \\
& \textbf{Ours} & $0.898\pm0.021$ & $0.72\pm0.11$ & $\mathbf{0.54}\pm0.48$ & $0.32\pm0.23$ & $1599$ \\

\midrule

\multirow{11}{*}{\parbox{2.3cm}{Quadratic-regularized objective}}
& Bonferroni box & $0.914\pm0.019$ & $0.86\pm0.17$ & $1.76\pm0.53$ & $0.27\pm0.13$ & $<0.01$ \\
& Conformal ($\ell_\infty$) & $0.902\pm0.016$ & $1.13\pm0.17$ & $2.24\pm0.57$ & $0.16\pm0.08$ & $<0.01$ \\
& Conformal ($\ell_2$) & $0.902\pm0.019$ & $1.05\pm0.07$ & $2.36\pm0.38$ & $0.33\pm0.08$ & $<0.01$ \\
& PCP & $0.896\pm0.023$ & $0.60\pm0.11$ & $\underline{1.38}\pm0.78$ & $0.42\pm0.21$ & $1.7$ \\
\cmidrule(lr){2-7}
& F-CROMS (10) & $0.898\pm0.019$ & $3.38\pm4.64$ & $2.19\pm0.43$ & $0.63\pm0.34$ & $1.1$ \\
& F-CROMS (20) & $0.897\pm0.022$ & $2.20\pm3.43$ & $1.89\pm0.54$ & $0.57\pm0.43$ & $2.2$ \\
& Yeh et al. & $0.900\pm0.016$ & $0.85\pm0.16$ & $1.87\pm0.54$ & $0.27\pm0.14$ & $1387$ \\
& Polyhedron & $0.902\pm0.020$ & $0.68\pm0.15$ & $2.03\pm0.84$ & $0.48\pm0.29$ & $<0.01$ \\
& Polyhedron (min size) & $0.902\pm0.020$ & $0.59\pm0.08$ & $1.94\pm1.06$ & $0.48\pm0.30$ & $0.04$ \\
\cmidrule(lr){2-7}
& Ours (w/o recal) & $0.906\pm0.022$ & $0.77\pm0.16$ & $1.50\pm0.48$ & $0.34\pm0.28$ & $1967$ \\
& \textbf{Ours} & $0.898\pm0.022$ & $0.73\pm0.12$ & $\mathbf{1.37}\pm0.40$ & $0.33\pm0.28$ & $1967$ \\

\midrule

\multirow{11}{*}{\parbox{2.3cm}{Newsvendor}}
& Bonferroni box & $0.914\pm0.019$ & $0.86\pm0.17$ & $8.57\pm0.38$ & $0.18\pm0.07$ & $<0.01$ \\
& Conformal ($\ell_\infty$) & $0.902\pm0.016$ & $1.13\pm0.17$ & $8.97\pm0.37$ & $0.19\pm0.07$ & $<0.01$ \\
& Conformal ($\ell_2$) & $0.902\pm0.019$ & $1.05\pm0.07$ & $7.95\pm0.28$ & $0.15\pm0.05$ & $<0.01$ \\
& PCP & $0.896\pm0.023$ & $0.60\pm0.11$ & $\mathbf{7.83}\pm0.39$ & $0.21\pm0.05$ & $2.0$ \\
\cmidrule(lr){2-7}
& F-CROMS (10) & $0.898\pm0.022$ & $1.12\pm0.25$ & $8.78\pm0.73$ & $0.19\pm0.07$ & $0.37$ \\
& F-CROMS (20) & $0.899\pm0.022$ & $1.13\pm0.26$ & $8.75\pm0.71$ & $0.21\pm0.10$ & $0.71$ \\
& Yeh et al. & $0.901\pm0.016$ & $0.99\pm0.23$ & $8.65\pm0.63$ & $0.19\pm0.08$ & $1782$ \\
& Polyhedron & $0.902\pm0.020$ & $0.68\pm0.15$ & $8.15\pm0.44$ & $0.19\pm0.05$ & $<0.01$ \\
& Polyhedron (min size) & $0.902\pm0.020$ & $0.59\pm0.08$ & $8.27\pm0.47$ & $0.21\pm0.06$ & $0.03$ \\
\cmidrule(lr){2-7}
& Ours (w/o recal) & $0.907\pm0.019$ & $0.81\pm0.12$ & $7.89\pm0.39$ & $0.16\pm0.05$ & $25$ \\
& \textbf{Ours} & $0.902\pm0.021$ & $0.78\pm0.14$ & $\underline{7.83}\pm0.41$ & $0.16\pm0.05$ & $25$ \\
\bottomrule
\end{tabular}}
\end{table}

The real-data energy-efficiency experiment in Table~\ref{tab:energy_efficiency_linear_auto} shows a similar pattern. Our method achieves the lowest worst-case cost among all methods with a coverage guarantee, while maintaining coverage near the target level within sampling variability. Compared with Ours (w/o recal), the recalibrated version has nearly identical regret and runtime, but provides the statistically valid two-stage construction used in the main method. Overall, these results support the conclusion that learning the geometry of the conformal set in a decision-aware manner improves robust downstream performance without sacrificing empirical coverage.

\begin{table}[!t]
\centering
\small
\caption{Energy Efficiency (real data). Mean\,$\pm$\,std over 80 seeds; $n=160$, $m=144$, $K=4$, $1{-}\alpha=0.90$. WC cost and regret are in units of $\times10^{-1}$; time is total wall-clock seconds per run. For Ours and Ours (w/o recal), time includes the 30-minute MIQCP attempt whenever it is made (Appendix~\ref{app:ours-impl}).}
\label{tab:energy_efficiency_linear_auto}

\resizebox{\linewidth}{!}{%
\begin{tabular}{lccccc}
\toprule
\textbf{Method} & \textbf{Coverage} & \textbf{Volume} & \textbf{WC cost} ($\times10^{-1}$) & \textbf{Regret} ($\times10^{-1}$) & \textbf{Time (s)} \\
\midrule
Bonferroni box & $0.910\pm0.029$ & $0.81\pm0.09$ & $2.79\pm1.10$ & $0.47\pm0.09$ & $<0.01$ \\
Conformal ($\ell_\infty$) & $0.907\pm0.026$ & $0.87\pm0.10$ & $3.86\pm1.09$ & $0.38\pm0.07$ & $<0.01$ \\
Conformal ($\ell_2$) & $0.905\pm0.027$ & $0.88\pm0.08$ & $3.39\pm1.04$ & $0.76\pm0.06$ & $<0.01$ \\
PCP & $0.906\pm0.032$ & $0.44\pm0.08$ & $3.02\pm1.11$ & $0.50\pm0.09$ & $1.2$ \\
\midrule
F-CROMS (10) & $0.907\pm0.032$ & $2.87\pm4.98$ & $2.75\pm1.15$ & $0.68\pm0.14$ & $0.23$ \\
F-CROMS (20) & $0.907\pm0.032$ & $2.93\pm7.53$ & $2.66\pm1.10$ & $0.64\pm0.13$ & $0.44$ \\
Yeh et al. & $0.906\pm0.037$ & $0.74\pm0.35$ & $3.60\pm1.32$ & $0.44\pm0.09$ & $408$ \\
Polyhedron & $0.902\pm0.037$ & $0.73\pm0.13$ & $4.33\pm1.15$ & $0.79\pm0.11$ & $<0.01$ \\
Polyhedron (min size) & $0.906\pm0.034$ & $0.77\pm0.17$ & $4.09\pm1.31$ & $0.71\pm0.20$ & $0.02$ \\
\midrule
Ours (w/o recal) & $0.894\pm0.036$ & $1.09\pm0.26$ & $\mathbf{2.31}\pm1.07$ & $0.56\pm0.10$ & $860$ \\
\textbf{Ours} & $0.900\pm0.033$ & $1.15\pm0.35$ & $\underline{2.36}\pm1.09$ & $0.56\pm0.10$ & $860$ \\
\bottomrule
\end{tabular}}
\end{table}

\end{document}